\documentclass[twoside]{article}
\usepackage[accepted]{aistats2025}
\usepackage[round]{natbib}

\usepackage{hyperref}
\usepackage{url}
\usepackage[utf8]{inputenc} % allow utf-8 input
\usepackage[T1]{fontenc}    % use 8-bit T1 fonts
\usepackage{hyperref}       % hyperlinks
\usepackage{url}            % simple URL typesetting
\usepackage{booktabs}       % professional-quality tables
\usepackage{amsfonts}       % blackboard math symbols
\usepackage{mathtools}      % amsmath with fixes and additions
\usepackage{amssymb}
\usepackage{nicefrac}       % compact symbols for 1/2, etc.
\usepackage{microtype}      % microtypography
\usepackage{xcolor}         % colors
\usepackage{siunitx}        % for proper typesetting of numbers and units
\usepackage{tikz}           % nice language for creating drawings and diagrams
\usepackage{graphicx}       
\usepackage{subcaption}     % \usepackage{subfigure}
\usepackage{amsthm}
\usepackage{thmtools, thm-restate}
\declaretheorem{theorem,definition,corollary,lemma,proposition,example,assumption}

\usepackage{algorithm}
\usepackage[noend]{algpseudocode}
\algnewcommand{\LeftComment}[1]{\Statex \(\triangleright\) #1}
\newcommand{\bA}{\mathbf{A}}

\newcommand{\bS}{\mathbf{S}}

\newcommand{\bX}{\mathbf{X}}
\newcommand{\bY}{\mathbf{Y}}

\newcommand{\ba}{\mathbf{a}}

\newcommand{\bs}{\mathbf{s}}
\newcommand{\by}{\mathbf{y}}

\newcommand{\cA}{\mathcal{A}}

\newcommand{\cF}{\mathcal{F}}

\newcommand{\cM}{\mathcal{M}}

\newcommand{\cQ}{\mathcal{Q}}

\newcommand{\cS}{\mathcal{S}}

\usepackage{nicefrac}
\begin{document}
\runningtitle{Q-function Decomposition with Intervention Semantics with Factored Action Spaces}
\runningauthor{Junkyu Lee, Tian Gao, Elliot Nelson, Miao Liu, Debarun Bhattacharjya, Songtao Lu}
\twocolumn[

\aistatstitle{Q-function Decomposition with Intervention Semantics for\\Factored Action Spaces}

\aistatsauthor{
Junkyu Lee$^*$ \And Tian Gao$^*$ \And Elliot Nelson$^{\dagger}$
}
\aistatsauthor{
Miao Liu$^*$ \And Debarun Bhattacharjya$^*$ \And Songtao Lu$^{\ddagger}$
}
\aistatsaddress{$^*$IBM T. J. Watson Research Center \qquad
\texttt{\{junkyu.lee,miao.liu1\}@ibm.com}, \texttt{\{tgao,debarunb\}@us.ibm.com}\\
$^{\dagger}$Independent \texttt{elliot137@gmail.com} \qquad $^{\ddagger}$The Chinese University of Hong Kong \texttt{stlu@cse.cuhk.edu.hk}
}
]
\begin{abstract}
Many practical reinforcement learning environments have a discrete factored action space that induces a large combinatorial set of actions, thereby posing significant challenges. 
Existing approaches leverage the regular structure of the action space and resort to a linear decomposition of Q-functions, which avoids enumerating all combinations of factored actions. 
In this paper, 
we consider Q-functions defined over a lower dimensional projected subspace of the original action space, and study the condition for the unbiasedness of decomposed Q-functions using causal effect estimation from the no unobserved confounder setting in causal statistics. 
This leads to a general scheme which we call action decomposed reinforcement learning that 
uses the projected Q-functions to approximate the Q-function in standard model-free reinforcement learning algorithms. The proposed approach is shown to improve sample complexity in a model-based reinforcement learning setting. We demonstrate improvements in sample efficiency compared to state-of-the-art baselines in online continuous control environments and a real-world offline sepsis treatment environment.
\end{abstract}

\section{INTRODUCTION}
Reinforcement learning (RL) combined with deep learning has advanced to achieve superhuman performance in many application domains \citep{mnih2015human,silver2017mastering}, 
but there is still significant room for improving sample complexity and computational tractability
for wider acceptance and deployment in real-world applications \citep{10.1145/3241036,scholkopf2022causality}. 
In many practical applications such as healthcare domains, 
it is preferable to collect a batch of interaction data in an off-policy manner or even in an offline setting,
although it limits collecting diverse and large amount of samples
due to the costly or infeasible nature of interactions \citep{komorowski2018artificial,tang2022leveraging}. 
Even in online environments, a structured combinatorial action space is well known to deteriorate sample efficiency significantly \citep{dulac2015deep,tavakoli2018action,tavakoli2020learning}.

The goal of this paper is to 
improve the sample efficiency of value-based RL algorithms 
for solving problems having a large factored action space.
The challenges for handling large action spaces are well recognized and 
typical approaches involve 
either decomposing the action space \citep{tang2022leveraging,rebello2023leveraging}
or augmenting data
\citep{pitis2020counterfactual,pitis2022mocoda,tang2023counterfactual}.
In multi-agent RL \citep{sunehag2018value,son2019qtran,wang2020qplex,rashid2020monotonic},
there exists a clear decomposable structure that
allows representing the global value function
as a combination of local value function per each agent.
However,
in single-agent RL, 
existing approaches are motivated by the assumption 
that the given MDP can be separated into independent MDPs
leading to a linear Q-function decomposition
\citep{russell2003q,tang2022leveraging,seyde2022solving}.

In factored action spaces,
each action $\ba$ is defined over multiple action variables $\bA=[A_1, \ldots, A_K]$,
where it is often the case that 
the effects of action subspaces
defined over the disjoint action variables
such as $\{A_1, A_2\}$ and $\{A_3, A_4\}$ may not interact with one another.
For example, in control problems over a 2-dimensional plane
with state variables comprised of position and velocity for each dimension,
the effect of an impulse impacting an object is separated per dimension
due to the modularity of underlying causal mechanisms.
\begin{figure*}
\centering
\begin{subfigure}[b]{0.25\textwidth}
  \includegraphics[width=0.99\textwidth]{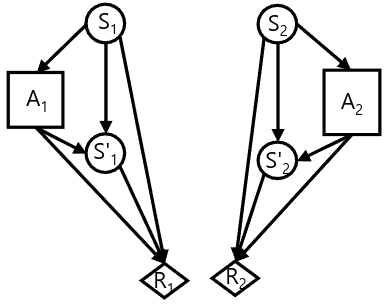}
  \caption{Fully Separable MDPs}
  \label{f11}
\end{subfigure}
\hfill
\begin{subfigure}[b]{0.25\textwidth}
  \includegraphics[width=0.99\textwidth]{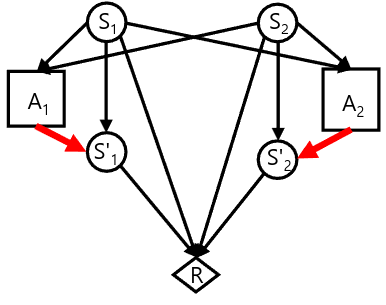}
  \caption{Separable Effects}
  \label{f12}
\end{subfigure}
\hfill
\begin{subfigure}[b]{0.25\textwidth}
  \includegraphics[width=0.99\textwidth]{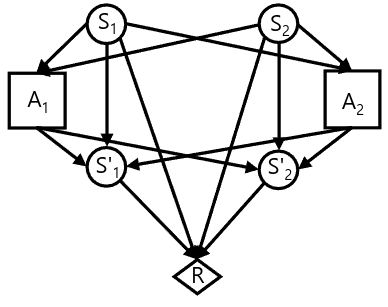}
  \caption{Non-separable MDP}
  \label{f13}
\end{subfigure}
\caption{
Decomposable Structures in Factored MDPs.
The diagrams show factored MDPs, where
the circles, squares, and diamonds represent
state, action, and reward variables.
\textbf{Fully Separable Structure}:
Fig \ref{f11} shows a factored MDP that can be fully separable
into two independent MDPs, considered in the previous work \citep{tang2022leveraging,seyde2022solving}.
\textbf{Separable Effects}:
Fig \ref{f12} shows a factored MDP that has non-separable dynamics
and rewards. However, the effects of factored actions are non-interacting,
and we study this structure with intervention semantics.
\textbf{Non-separable Structure}:
Fig \ref{f13} shows a non-separable factored MDP.
}
\label{fig:structures}
\end{figure*}
Motivated by such a modular structure in factored action spaces,
we investigate Q-function decomposition with the intervention semantics, 
leading to a general decomposition scheme that leverages the projected action spaces. 

In Section 3, 
we study theoretical properties around the soundness and sample complexity 
of Q-function decomposition in factored action spaces under the tabular model-based RL setting.
In Section 4,
we present a practical scheme called action decomposed RL
that augments model-free algorithms \citep{mnih2015human,fujimoto2019off}
with an improved critic learning procedure based on Q-function decomposition.
In Section 5, we implement this scheme with
Deep Q-networks (DQN) \citep{mnih2015human,van2016deep} and 
batch constrained Q-learning (BCQ) \citep{fujimoto2019off}
for experiments involving online 2D control environments
and sepsis treatment offline environments derived from the real world MIMIC-III dataset,
demonstrating improved sample efficiency with our proposed approach.

\section{PRELIMINARIES}
We consider reinforcement learning (RL) environments 
as factored Markov decision processes (MDPs) 
described by a tuple 
$\cM:=\langle\cS, \cA, P^0, P, R, \gamma \rangle$,
where
the state space $\cS$ and the action space $\cA$
are  factored into 
a set of variables $\bS:=\{S_1, \ldots, S_M\}$
and 
$\bA:=\{A_1, \ldots, A_N\}$.
We denote states and actions by vectors
$\bs = [s_1, \cdots, s_M]$ 
and 
$\ba = [a_1, \cdots, a_N]$, respectively.
$P^0:=\cS \rightarrow [0, 1]$ is the initial state distribution,
$P:=\cS \times \cA \times \cS \rightarrow [0, 1]$ 
is the state transition function,
$R:=\cS \times \cA \times \cS \rightarrow \mathbb{R}$ is the reward function,
and 
$\gamma \in (0, 1]$ is a discounting factor.

RL agents find an optimal policy
$\pi^{*} \in \{\pi: \cS \rightarrow \cA\}$
that maximizes  the value function,
$V_{\pi}(\bs)= 
\mathbb{E}_{\pi}[ \sum_{t=0}^{\infty} \gamma^t 
R(\bs^t, \ba^t, \bs^{t+1})\!\! \mid\!\! \bs^0\!\!=\!\!\bs]$.
Off-policy Q-learning algorithms 
sample state transition tuples
$(\bs, \ba, r, \bs')$
using a behavior policy $\pi^b:\cS \times \cA \rightarrow [0,1]$, 
and learns Q-function
$Q_{\pi}(\bs,\ba)=
\mathbb{E}_{\pi}[ \sum_{t=0}^{\infty} \gamma^t 
R(\bs^t, \ba^t, \bs^{t+1})\!\! \mid\!\! \bs^0\!\!=\!\!\bs, \ba^0=\ba]$
to find the optimal deterministic policy $\pi^*(\bs)$ by
$\ba^* = \arg\max_{\ba \in \cA} Q_{\pi^*}(\bs, \ba)$.

In this paper, we study Q-functions with intervention semantics
in the `no unobserved confounder setting'
\citep{pearl2009causality,schulte2024online}.
Given states $\bs$ and $\bs'$
at the current and the next time steps,
a causal model prescribes the dynamics 
by structural equations 
$S' \leftarrow F_{S'}\big(pa(S'), U_{S'}\big)$
defined over all variables $S' \in \bS'$,
where 
$U_{S'}$ is 
an exogenous noise variable and $pa(S') \subset \bS$.
In the absence of intervention,
the structural equations
factorize the state transition function as,
$P(\bS'|\bS) = \prod_{m=1}^{M}P\big(S'_{m}|pa(S'_{m})\big)$,
which we call
\textit{no-op} dynamics, 
governing the transition of the state variables 
free from the direct effect of $do(\bA)$\footnote{We use $do$ operator to emphasize the intervention semantics of an action $\bA$.}.

Action $do(\bA)$ fixes the value 
of the next state variables $\text{Eff}(\bA) \subset \bS'$ 
subject to 
the current state variables $\text{Pre}(\bA) \subset \bS$
\footnote{We use terminology from automated planning where $\text{Eff}(\bA)$ and $\text{Pre}(\bA)$ refer to the effect and precondition of action $\bA$.
$\text{Eff}(\bA)$ are the state variables that are controllable by action,
whereas $\bS'\setminus \text{Eff}(\bA)$ are non-controllable state variables.
The intervention policy $\sigma_{\bA}$ can be depend on the state variables $\text{Pre}(\bA)$. 
}.
Namely,
for all $S' \in \text{Eff}(\bA)$,
the structural equation $F_{S'}$
is replaced by a conditional intervention policy $\sigma_{\bA}$
such that 
\small
\begin{equation}
\text{Eff}(\bA) \leftarrow \sigma_{\bA}\big(\text{Pre}(\bA)\big).    
\end{equation}
\normalsize
% $\text{Eff}(\bA) \leftarrow \sigma_{\bA}\big(\text{Pre}(\bA)\big)$.
Then, the interventional state transition function
$P\big(\bS'|\bS,do(\bA)\big)$
follows a truncated factorization as,
\small
\begin{equation}
\label{eq:state_transition_causal}
P\big(\bS'|\bS,do(\bA)\big)\!=\!\!
P(\bS'\setminus\!\text{Eff}(\bA)|\bS)
\cdot
\mathbb{I}\Big[
\text{Eff}(\bA)\!=\!
\sigma_{\bA}\big(\text{Pre}(\bA)\big)
\Big],
\end{equation}
\normalsize
where $\mathbb{I}$ is the indicator function.

The deterministic reward can be modeled as
a structural equation
$R \leftarrow F_{R}\big(pa(R)\big)$,
where $pa(R) \subset \bS \cup \bS'$
and the domain of $R$, $dom(R)$ is the range of the reward function
\footnote{We abused the notation.
$R$ denotes either a random variable that defines the reward function with a structural equation $F_{R}$
or a reward function when it is clear from context.}
.
The potential outcome \citep{rubin1974estimating} 
of $R$ subject to action $do(\bA)$
is prescribed by $R \leftarrow F_{R}\big(pa(R), do(\bA)\big)$
and 
the conditional expected reward 
$\mathbb{E}[R(\bs, \ba, \bs')|\bs, \ba]$
can be written as a causal effect,
\small
\begin{equation}
\label{eq:expected_reward_causal}
\sum_{\bS'}P\big(\bS'|\bS, do(\bA)\big)
\!\!\!\!\!\!\sum_{r \in dom(R)}\!\!\!\!
r \cdot \mathbb{I}[
r=F_{R}\big(pa(R), do(\bA)\big)
].
\end{equation}
\normalsize

In terms of the Neyman-Rubin potential outcomes frameworks
\citep{rubin2005causl},
the state $\bS$ is the observed confounder
since it influences both treatment on $\bS'$ %by $do(\bA)$
and the outcome $R$.
The primary interest in causal statistics 
is to estimate causal effects
in the presence of confounding bias,
which stems from the fundamental problem of causal inference.
If RL agents 
could revisit the same states and collect alternative action choices,
then we have no issue with the confounding bias in principle.
However, 
we will see that
formulating RL as 
causal effect estimation with intervention semantics
offers an opportunity to improve the sample efficiency.
There are similar active research efforts to improve sample efficiency 
in bandits \citep{johansson2016learning,saito2023offpolicy}
and machine learning applications such as recommendation systems \citep{gao2024causal}.

We can categorize the decomposable structures in factored MDPs
by considering which components are separable, as shown in Figure \ref{fig:structures}.
\cite{tang2022leveraging} studied the case  shown in Figure \ref{f11},
where the state space and the reward function are all separable per projected action spaces.
In such a case, Q-function can be exactly decomposed as
\small
\begin{equation}
\label{eq:linear_q_dec}
\cQ_{\pi}(\bS,\bA) = \sum_{n=1}^{N}Q^n(\bS, A_n),    
\end{equation}
\normalfont
% $\cQ_{\pi}(\bS,\bA) = \sum_{n=1}^{N}Q^n(\bS, A_n)$,
where each $Q^n(\bS, A_n)$ is defined over the whole state space, yet 
restricted to apply actions in the subspace $\cA_n$.
\cite{rebello2023leveraging} proposed an importance sampling estimator that leverages the fully separable MDP structure.
\cite{seyde2022solving} 
demonstrated state-of-the-art performance for solving continuous control problems
through a simple modification that replaces Q-function with a 
linear decomposition of Q-functions as shown in Eq.~\eqref{eq:linear_q_dec}
in deep Q-learning (DQN) algorithms \citep{mnih2015human}.

\section{METHODS}
\subsection{Projected Action Space MDPs}
Let's consider the decomposition case 
shown in Figure \ref{f12},
where
we assume 
that the effect of factored actions are non-interacting.
\begin{assumption}
\normalfont
\label{assumption:1}
Given a partition of action variables $\bA=[\bA_1, \ldots, \bA_K]$,
there exists a partition over state variables 
$\bS'=[\bS'_1, \ldots, \bS'_K, \bS'_{K+1}]$ 
such that $\bS'_k=\text{Eff}(\bA_k)$
and 
$\bS'_{K+1}\cap \text{Eff}(\bA)
=\emptyset$,
where
$\text{Eff}(\bA)
=\cup_{k=1}^{K} \text{Eff}(\bA_k)
$.
\end{assumption}
Then, we can factorize the interventional state transition function
$P\big(\bS'|\bS, do(\bA)\big)$ as
\small
\begin{equation}
\label{eq:intervention_transition}
\!\!P\big(\bS'|\bS, do(\bA)\big)\!\!=\!\!
P(\bS_{K+1}'|\bS,\text{Eff}(\bA))\!\!
\prod_{k=1}^{K}\!\!P\big(\bS'_k|\bS, do(\bA_k)\big).
\end{equation}
\normalsize

In terms of the interventional state transition function $P\big(\bS'|\bS, do(\bA)\big)$,
$Q_{\pi}(\bs, \ba)$ can be written as
\small
\begin{equation}
\label{eq:q_function_causal}
    Q_{\pi}(\bs, \ba)\!=\!
    \sum_{\bs'}\!\!P\big(\bs'|\bs, do(\ba)\big)
    [
    R(\bs, \ba, \bs')\!+\!\gamma Q_{\pi}\big(\bs', \pi(\bs')\big) 
    ].
\end{equation}
\normalsize

For each of the projected action space $\cA_k$ subject to action variables $\bA_k$,
let's define a projected action space MDP as follows.

\begin{definition}[Projected Action Space MDP]
\normalfont
Given a factored MDP $\cM:=\langle \cS, \cA, P^0, P, R, \gamma\rangle$,
and 
a projected action space $\cA_k:=\bigtimes_{A_i \in \bA_k} A_i$,
a projected action space MDP 
$\cM^{k}:=\langle \cS, \cA_k, P^0, P, R, \gamma\rangle$ 
is defined by 
the state transition function
in terms of the interventional state transition function
$P\big(\bS'|\bS, do(\bA_k)\big)$ for a subset of action variables $\bA_k$ as,
\small
\begin{gather}
\label{eq:factored_state_transition_causal}
\!\!\!\!P\big(\bS_k'|\bS, do(\bA_k)\big)
P\big(\bS_{K+1}'|\bS,\text{Eff}(\bA)\big) 
\!\!\!\!\!\!\!\!
\prod_{i \in [1..K], i \neq k}
\!\!\!\!\!\!\!
P(\bS_i'|\bS)
,
\end{gather}
\normalsize
where
$P\big(\bS'_k|\bS, do(\bA_k)\big)$
is the interventional distribution 
over $\bS'_k$ under the effect of $do(\bA_k)$.
\end{definition}

Under Assumption \ref{assumption:1},
we can rewrite the state transition function 
$P\big(\bS'|\bS,do(\bA)\big)$
with $P\big(\bS'|\bS,do(\bA_k)\big)$ as,
\small
\begin{align}
P\big(\bS'|\bS, do(\bA)\big)\!\!=
\!\!
P\big(\bS'|\bS, do(\bA_k)\big)
\!\!\!\!\!\!\!
\prod_{i=1,i\neq k}^{K}
\!\!\!\!\!\!\!
\frac{P\big(\bS'_i|\bS, do(\bA_i)\big)}{P(\bS'_i|\bS)}\label{eq:intervention_factorization_b}\\
=
P\big(\bS'|\bS, do(\bA_k)\big)
\!\!\!\!
\prod_{i=1,i\neq k}^{K}
\!\!\!\!
\frac{\mathbb{I}
\big[
    \bS'_i=\sigma_{\bA_i}\big(\text{Pre}(\bA_i)\big)
\big]
}
{P(\bS'_i|\bS)}\label{eq:intervention_factorization_c}\\
=
P\big(\bS'|\bS, do(\bA_k)\big)
\!\!\!\!\!\!\prod_{i=1,i\neq k}^{K}\!\!
1^{}\big/
{
    P\Big(
        \bS'_i=\sigma_{\bA_i}\big(\text{Pre}(\bA_i)\big)
    |\bS
    \Big)
},\label{eq:intervention_factorization_d}
\end{align}
\normalsize
where 
$\sigma_{\bA_i}\big(\text{Pre}(\bA_i)\big)$
is the conditional intervention policy that fixes
the value of $\text{Eff}(\bA_i)$.

Proposition \ref{prop:projected_q_function}
shows that 
Q-function $Q_{\pi_k}(\bs, \ba_k)$ in $\cM^k$
can be written 
in terms of $P\big(\bS'|\bS, do(\bA_k)\big)$.
% \begin{proposition}[Projected Q-function]
\begin{restatable}[Projected Q-function]{proposition}{propqfunction}
\label{prop:projected_q_function}
\normalfont
Given a projected MDP $\cM^{k}$ over action variables $\bA_k$,
the Q-function 
$Q_{\pi_{k}}(\bs, \ba_k)$
can be recursively written as,
\small
\begin{gather*}
\label{eq:projected_q_function_causal}
    Q_{\pi_k}(\bs, \ba_k)\!\!=
    \!\!\sum_{\bs'}\!\!
    P\!\big(\bs'|\bs, do(\ba_k)\big)
    \big[R(\bs, \ba_k, \bs')
    \!+\!\gamma 
    Q_{\pi_{k}}(\bs', 
    \pi_k(\bs'))
    \big],
\end{gather*}
\normalsize
where
$\pi_{k}:\cS \rightarrow \cA_k$
is a factored policy for $\bA_k$.
% \end{proposition}
\end{restatable}
\begin{proof}
The actions in $\cA_k$ only intervene on $\bS'_k$ and 
the rest of the state variables $\bS'\setminus \text{Eff}(\bA_k)$
follow the \textit{no-op} dynamics.
Namely, the state transition follows 
Eq.\!~\eqref{eq:factored_state_transition_causal}.
By definition, 
$Q_{\pi_k}(\bs, \ba_k)$
is the value of 
applying action 
$do(\bA^0\!\!=\!\!\ba_k)$ in state $\bS^0\!\!\!=\!\!\bs$ at time step $t\!\!\!=\!\!\!0$,
and applying actions $do\big(\bA^t\!\!\!\!=\!\!\pi_k(\bS^t)\big)$ 
for the remaining time steps $t\!\!=\!\!1..\infty$.
It is easy to rewrite
$Q_{\pi_k}(\bs, \ba_k)\!=\!
\sum_{\bS^1}P(\bS^1|\bs, do(\ba_k))R(\bs, \ba_k, \bS^1)
\!\!+
\sum_{t=1}^{\infty}\sum_{\bS^1..\bS^{t}}
\prod_{j=0}^{t} \gamma^{t} P(\bS^{j+1}|\bS^j, do(\pi_k(\bS^j)))
R(\bS^t,$ $\pi_k(\bS^t),\bS^{t+1})$
as desired.
\end{proof}
From Eq.~\eqref{eq:intervention_factorization_b},
we see that 
the interventional state transition functions
$P\big(\bS'|\bS, do(\bA)\big)$
and
$P\big(\bS'|\bS, do(\bA_k)\big)$
are different only in $P(\bS'_i|\cdot)$ for 
$i \neq k$
such that we can weight the projected 
Q-functions $Q_{\pi_k}(\bs, \ba_k)$ to represent 
$Q_{\pi}(\bs, \ba)$.
\begin{definition}[Weighted Projected Q-functions]
\normalfont
Given a policy $\pi$ and its projected policy $\pi_k$,
a weighted projected Q-function $\tilde{Q}_{\pi}(\bs, \ba_k)$
can be defined 
as,
\small
\begin{align}
\tilde{Q}_{\pi_k}(\bs, \ba_k)
&
=
\sum_{\bs'}
\frac
{
    P\big(\bs'|\bs, do(\ba_k)\big)
}
{
    \rho_{-k}(\bs, \bs')
}
\big[
R(\bs, \ba_k, \bs')
+\nonumber\\
&\gamma
\tilde{Q}_{\pi_k}\big(\bs', \pi_k(\bs')\big)
\big]
\mathbb{I}\big[\text{Eff}(\bA)=\sigma_{\bA}\big(\text{Pre}(\bA)\big)\big],
\end{align}
\normalsize
where 
$\rho_{-k}(\bs, \bs')\!\!=\!\! 
\prod_{i=1,i\neq k}^{K}
P\big(
\bs'_i=\sigma_{\ba_i}(\text{Pre}(\ba_i))
|\bs
\big)
$
from Eq.~\eqref{eq:intervention_factorization_d},
and the values of the next state variables $\text{Eff}(\bA)$ are consistent to
action $do\big(\pi(\ba')\big)$.
\end{definition}
Note that 
there are two main differences between 
$\tilde{Q}_{\pi_k}(\bs', \ba'_k)$
and 
${Q}_{\pi_k}(\bs', \ba'_k)$.
First, the weighted projected Q-functions
are defined relative to the action trajectories 
that follow the policy $\pi$ in the action space $\cA$.
Second, 
the weights $\rho_{-k}(\bs, \bs')$ represent 
the propensity score \citep{rubin2005causl} 
of 
the \textit{no-op} dynamics.

\subsection{Model-based Factored Policy Iteration}
Next, we present a concrete model-based RL algorithm, called 
model-based factored policy iteration (MB-FPI)
for studying the soundness of proposed Q-function decomposition.
We also show improved sample complexity of MB-FPI under the following assumptions.
\begin{assumption}
\label{assumption:2}
\normalfont
The \textit{no-op} dynamics
$P(\bS'|\bS)=\prod_{k=1}^{K}\!P(\bS'_k|\bS)
P\big(\bS'_{K+1}|\bS,\text{Eff}(\bA)\big)$ 
is positive. %a positive distribution.
\end{assumption}
\begin{assumption}
\label{assumption:3}
\normalfont
Given a behavior policy $\pi^b(\ba|\bs)$,
and its factored policy $\pi^b_k(\ba_k|\bs)$,
the supports of $P(\bs'|\bs)$
induced by $\pi^b$ and $\pi^b_k$ are the same in every states $\bs$,
namely,
$
\big\{\bs' | 
\sum_{\ba}P(\bs'|\ba, \bs)\pi^b(\ba|\bs) > 0 \big\}
=
\big\{\bs' | 
\sum_{\ba_k}P(\bs'|\ba_k, \bs)\pi^b_k(\ba_k|\bs) > 0 \big\}
$.
\end{assumption}
In words, the above assumptions ensure that
the exploration with projected actions in $\cA_k$
covers the same state space visited by the action space $\cA$.

\begin{algorithm}[t]
    \caption{Model-based Factored Policy Iteration}
    \label{alg:mb-fpi}
    \begin{footnotesize}
    \begin{algorithmic}[1]
        \State Learn 
               intervention policies $\sigma_{\bA_k}$ for all $k = 1 to K$, 
               the factored \textit{no-op} dynamics 
               $P\big(\bS'_{K+1}|\bS,\text{Eff}(\bA)\big)$,
               and the reward function $R(\bS, \bA, \bS')$.
        \State Initialize 
        an arbitrary factored policy $\pi=[\pi_1, \ldots, \pi_K]$
        \Repeat
            \Statex{\phantom{xx} \textsc{//Policy Evaluation}}
            \For{k = 1 to K}
                \State
                $P(\bS'|\bS, \bA_k)
                \leftarrow
                \mathbb{I}[\bS'_k
                =
                \sigma_{\bA_k}(\bS)]
                \prod_{i=1, i\neq k}^{K} \mathbb{I}[\bS'_i=
                \phantom{xxxxxxxxxxxxxxxxxx} \sigma_{\pi_i(\bS)}(\bS)]
                P\big(\bS'_{K+1}|\bS,\text{Eff}(\bA)\big)
                $
                \For{each $(\bs, \ba_k) \in (\bS, \bA_k)$}
                    \State
                    $\tilde{Q}_{\pi_k}(\bs, \ba_k)\!\!\leftarrow\!\!
                    \sum_{\bS'}P(\bS'|\bs, \ba_k) [
                    R(\bs, \ba_k, \bS') +
                    \phantom{xxxxxxxxxxxxxxxxxxx}
                    \gamma \tilde{Q}_{\pi_k}(\bs, \ba_k)]
                    $
                \EndFor
            \EndFor
      
            \Statex{\phantom{xx} \textsc{//Policy Improvement}}
            \State Select an arbitrary $k$ for policy improvement
            \For {each state $\bs \in \cS$}
                \State $\pi_k(\bs) \leftarrow \arg\max_{\ba_k} \tilde{Q}_{\pi_k} (\bs, \ba_k)$
            \EndFor
            
        \Until No change in $\pi(\bs)$
        
    \end{algorithmic}
    \end{footnotesize}
\end{algorithm} 
Algorithm \ref{alg:mb-fpi}
is a model-based algorithm 
that learns models of projected MDPs $\cM^k$ (line 1),
and
performs 
policy iteration \citep{sutton2018reinforcement} 
to find a locally optimal policy
$\pi^*(\bs)=[\pi^*_1(\bs), \ldots, \pi^*_K(\bs)]$.
The policy evaluation procedure 
computes
state transition function 
$P(\bS'|\bS, \bA_k)$ (line 5)
and updates Q-tables $\tilde{Q}_{\pi_k}(\bs, \ba_k)$ (line 7).
The policy improvement procedure
selects an arbitrary $\pi_k$ to improve (line 11).
The algorithm terminates and returns 
the locally optimal policy $\pi^*$
when there is 
no change during updating $\pi$.

Under the intervention semantics,
$\sigma_{\bA_k}$ fixes the value of $\bS'_k$, and 
we see that $P(\bS'|\bS, \bA_k)$ constrains the state trajectory to follow 
not only a single factored policy $\pi_k$, but also all other policies
following the factorization shown in Eq.\!~\eqref{eq:intervention_transition}.
As a result, $\tilde{Q}_{\pi_k}(\bs, \ba_k)$
maintains the values of $\ba_k$ in $\bs$
while fixing all the rest of the actions $\ba_i$ 
to follow the current policy $\pi_i$
for all $i \in [1..K]$ except $i = k$.
This implies the consistency between two value functions 
that are evaluated on a deterministic policy $\pi$ and its projection $\pi_k$,
namely, 
$\tilde{V}_{\pi}(\bs) = \tilde{V}_{\pi_k}(\bs)$,
where
$\tilde{V}_{\pi}(\bs)=\tilde{Q}_{\pi}\big(\bs, \pi(\bs)\big)$
and
$\tilde{V}_{\pi_k}(\bs)=\tilde{Q}_{\pi_k}\big(\bs, \pi_k(\bs)\big)$.
However, 
it does not guarantee to find the globally optimal policy.
The following Theorem \ref{theorem:convergnce_fpi}
shows that MB-FPI converges to the globally optimal policy
under the additional monotonicity assumption on $Q_{\pi}(\bs, \ba)$.
\begin{restatable}[Convergence of MB-FPI]{theorem}{convergncefpi}
% \begin{theorem}[Convergence of MB-FPI]
\label{theorem:convergnce_fpi}
\normalfont
MB-FPI converges 
to a locally optimal policy in a finite number of iterations.
If the underlying $Q_{\pi}(\bs, \ba)$ is monotonic,
then MB-FPI converges to the optimal policy $\pi^*$.
% \end{theorem}
\end{restatable}
\begin{proof}
The convergence of MB-FPI follows from the fact that the policy improvement step 
changes the factored policy $\pi=[\pi_1,\ldots,\pi_k]$
per block-wise update of the action variables  in a finite space MDP (line 8).
If the underlying $Q_{\pi}(\bs, \ba)$ is monotonic \citep{rashid2020monotonic},
namely, the block-wise improvement is equivalent to the joint improvement,
$Q_{\pi_k}\big(\bs, \pi_k(\ba_k)\big)
> Q_{\pi_k}(\bs, \ba_k)
\iff 
Q_{\pi}\big(\bs, \pi_k(\ba)\big)
> Q_{\pi}(\bs, \ba)
$,
the policy improvement step reaches the same fixed point that would have reached by
the non-factored policy iteration.
The full proof can be found in the supplementary material.
\end{proof}
\begin{restatable}[Sample Complexity of MB-FPI]{theorem}{complexitymbfpi}
% \begin{theorem}[Sample Complexity of MB-FPI]
\normalfont
Given $\delta \in (0,1)$ and $\epsilon > 0$,
the sample complexity for learning
the \textit{no-op} dynamics $P(\bS'_{K+1}|\bS, \text{Eff}(\bA))$ 
and
the intervention policy $\sigma_{\bA_k}(\text{Pre}(\bA_k)$ 
within error bound $\epsilon$ with at least probability $1-\delta$
are 
$$N_{P} \geq \frac{|\bS_{K+1}| |\bS| |\bS\setminus \bS_{K+1}|}{\epsilon^2} \log \frac{2 |\bS||\bS\setminus \bS_{K+1}|}{\delta}$$
and 
$$N_{\sigma} \geq \frac{|\text{Eff}(\bA_k)| |\text{Pre}(\bA_k)|}{\epsilon^2} \log \frac{2|\text{Pre}(\bA_k)|}{\delta}.$$
% \end{theorem}
\end{restatable}
\begin{proof}
Following a basic sample complexity analysis of model-based RL
that utilizes Hoeffding's inequality and the union bounds \citep{agarwal2019reinforcement},
we can obtain the sample complexity
for learning a probability distribution $P(\bX|\bY)$ as,
$$N \geq {|\bX||\bY|}/{\epsilon^2} \cdot \log ({2 |\bY|}/{\delta}).$$
Then, the desired results can be obtained by 
adjusting the sets $\bX$ and $\bY$ appropriately,
i.e., $\bX = \bS'_{K+1}$ and $\bY = \bS \cup \text{Eff}(\bA)$.
The full proof can be found in the supplementary material.
\end{proof}
$N_p$ and $N_{\sigma}$ 
highlight the improved sample complexity
due to the factored MDP structure as well as
the non-interacting effects between projected action spaces,
compared with the sample complexity 
of the full dynamics model of the non-separable MDPs,
$N \geq |\bS|^2|\bA|/\epsilon^2 \cdot \log (2 |\bS||\bA|/\delta)$.

\section{ACTION DECOMPOSED RL FRAMEWORK}
% [Practical issues; knowing interventions, 
% and model-free approach adaptations
% ]
In the previous section, 
we investigated theoretical properties 
around Q-function decomposition in factored action spaces
and its advantage in a tabular model-based RL setting,
which implicitly assumes 
that the intervention policies can be learned in 
a well-structured discrete state space 
such that we know how the state variables change
under the influence of actions.

In practical environments, 
such assumptions don't hold  if the state space is continuous, unstructured, or high-dimensional.
Therefore, we take the essential ideas developed in 
the previous section, 
and
present a practical and general scheme that augments two components 
for critic learning in model-free value-based RL algorithms.

First, we use 
the projected Q-functions $Q_{\pi_k}(\bs, \ba_k)$
% in projected MDPs
as a basis 
for approximating
the Q-function $Q_{\pi}(\bs, \ba)$.
Let's consider a function class 
$\cF$ 
be comprised of the 
projected Q-functions $Q_{\pi_k}(\bs, \ba_k)$,
$$\cF:=\{ \cQ_{\pi_k}(\bs, \ba_k) \mid k \in [1..K]\}.$$
A linear combination of the projected Q-functions
can be written as
$\tilde{Q}_{\pi}(\bs, \ba)
= 
\sum_{k=1}^{K} w_k Q_{\pi_k}(\bs, \ba_k)$
for some coefficients 
$\{w_k \mid k \in [1..K]\}$,
and 
more general combinations 
using neural networks can be written as
\small
\begin{equation}
\label{eq:q_dec_nn}
\tilde{Q}_{\pi}(\bs, \ba)
= 
F(Q_{\pi_1}\big(\bs, \ba_1), \ldots, Q_{\pi_K}(\bs, \ba_K)\big)
.    
\end{equation}
\normalsize
The second component is a data augmentation procedure 
that learns the dynamics models of the projected MDPs 
and the reward model and using them to generate samples 
while training individual projected Q-functions $Q_{\pi_k}(\bs, \ba_k)$.
Next, we present two concrete algorithms
that modify the critic learning steps in DQN and BCQ.

\begin{algorithm}[t]
  \caption{Action-Decomposed DQN}
  \label{alg:addqn}
  \begin{footnotesize}
  \begin{algorithmic}[1]
    \For {episode  = 1 to num-episodes}
        \State $\bs \leftarrow$ initial state from environment
        \For {step  = 1 to episode-length}
            \Statex{\phantom{xx} \textsc{//Collect Samples}}
            \If {random $< \epsilon$}
                \State draw a random action $\ba$
                \Statex \phantom{xxxxxxxx} with probability $p$, draw a projected action
            \Else
                \State select actions $[\ba_1,\ldots,\ba_K]$ from $Q_{\pi_k}(\bs, \ba_k)$
            \EndIf 
            \State Sample a transition $(\bs, \ba, \bs', r)$ from environment
            \State Store the sample to replay buffer $D$
            \Statex \phantom{xxxxx} if action was $\ba_k$, store the sample to buffer $D_k$ 
            \State $\bs \leftarrow \bs'$
            \Statex{\phantom{xx} \textsc{//Train Q-networks}}
            \State Sample a batch $B$ from buffer $D$
            \For{k= 1 to K}
                \State Modify $B$ to $B_k$ to follow $P(\bS'|\bS, do(\bA_k))$
                \Statex  \phantom{xxxxxxxxxx} using learned dynamics and reward model
                \State Train $Q_{\pi_k}(\bs, \ba_k)$ with $B_k$
            \EndFor
            \State Train $\tilde{Q}_{\pi}(\bs, \ba)$ with $B$
        \EndFor
        \State Update target network
        \Statex{\textsc{//Train Dynamics and Reward Models}}
        \State Sample a batch $B$ from $D$, train reward model
        \For{k = 1 to K}
            \State Sample a batch $B_k$ from $D_k$, train dynamics
        \EndFor
        
    \EndFor 
  \end{algorithmic}
  \end{footnotesize}
\end{algorithm}
\subsection{Action Decomposed DQN}
One of the main challenges in DQN-style algorithms 
is handling large discrete action spaces,
especially when the size of the action space
grows in a combinatorial manner.
In general, value decomposition approaches
introduce local value functions and 
those approaches perform action selection in a local manner
to avoid enumerating all possible combinations of actions.
Here, we leverage the projected Q-functions in factored action spaces
to improve the sample efficiency of DQN.

Algorithm \ref{alg:addqn} 
presents our proposed modification to DQN, 
which we call action decomposed DQN (AD-DQN).
Given a factored action space $\cA=\cA_1 \times \cdots \times \cA_K$,
the Q-network 
$\tilde{Q}_{\pi}(\bs, \ba)$
in AD-DQN 
combines $K$ sub Q-networks, 
each corresponding to $Q_{\pi_k}(\bs, \ba_k)$.
The $F$ is a mixer network \citep{rashid2020monotonic}
that implements Eq.~\eqref{eq:q_dec_nn}.
The mixer network could be a linear layer, or a non-linear MLP 
for combining 
the values of $Q_{\pi_k}(\bs, \ba_k)$ to yield the final outcome of $\tilde{Q}_{\pi}(\bs, \ba)$.
%[difference vs. MARL]

During training (line 11-15),
we first sample a mini-batch $B$ from a replay buffer $D$.
For training sub Q-networks $Q_{\pi_k}(\bs, \ba_k)$,
we modify the samples in $B$ to follow the transition dynamics 
$P\big(\bS'|\bS, do(\bA_k)\big)$ of the 
projected MDP $\cM^k$ with learned dynamics models of the projected MDPs (line 13).
Namely, for a transition $(\bs, \ba, \bs', r) \in B$,
we generate $(\bs, \ba_k, \tilde{\bs}', \tilde{r})$
by projecting $\ba$ to $\ba_k$, sampling 
$\tilde{\bs'} \sim P\big(\bS'|\bs, do(\ba_k)\big)$,
and  evaluating the learned reward model, 
$\tilde{r} = R(\bs, \ba_k, \tilde{\bs}')$.
Next,
we train the full network $\tilde{Q}_{\pi}(\bs, \ba)$ 
including the mixer $F$ that combines the values of $Q_{\pi_k}(\bs, \ba_k)$ (line 15).
While collecting samples or training the networks,
we select actions $\ba$ by concatenating  the projected actions $[\ba_1, \ldots, \ba_K]$ 
using the global Q-network,
$\ba_k \leftarrow \arg\max_{\bA_k} \tilde{Q}_{\pi}(\bs, \bA_k)$.

Note that the steps for training dynamics and reward models 
make AD-DQN a model-based RL algorithm (line 17-19). 
When we have access to the intervention policies,
the dynamics model can be further simplified to a single \textit{no-op}
dynamics model $P\big(\bS'|\bS, \text{Eff}(\bA)\big)$,
which is sufficient to generate $\tilde{\bs}'$.
However, 
it may be difficult to learn such intervention policies in practical RL environments.
Therefore, AD-DQN learns 
the dynamics model for the projected MDPs 
using the samples generated in the projected action spaces (line 5),
Then, it generates the synthetic data $(\bs, \ba_k, \tilde{\bs}', \tilde{r})$
from the samples $(\bs, \ba, \bs, r)$  obtained from the environment
to train $Q_{\pi_k}(\bs, \ba_k)$.

\begin{algorithm}[t]
  \caption{Action-Decomposed BCQ}
  \label{alg:adbcq}
  \begin{footnotesize}
  \begin{algorithmic}[1]
  \Statex{\textsc{//Train Dynamics and Reward Models}}
  \State Train a reward model with offline data $D$
  \For{k = 1 to K}
    \State Collect $D_k$ with projected actions $\ba_k$ from $D$
    \State Train dynamics model with $D_k$
  \EndFor
  
  \For{t = 1 to max-training}
    \Statex{\phantom{} \textsc{//Train Q-networks and Generative Models}}
    \State Sample a batch $B$ from $D$
    \For{k=1 to K}
        \State Modify $B$ to $B_k$ to follow $P(\bS'|\bS,do(\bA_k))$
        \State Train $Q_{\pi_k}(\bs, \ba_k)$ and $G_{k}(\ba_k|\bs)$ for BCQ
    \EndFor
    \State Train $Q_{\pi}(\bs, \ba)$ and $G(\ba|\bs)$ for BCQ
    \State Update target network
  \EndFor
  \end{algorithmic}
  \end{footnotesize}
\end{algorithm}
\subsection{Action Decomposed BCQ}
In off-policy RL approaches 
\citep{lange2012batch,levine2020offline,uehara2022review}, 
a behavior policy generates a batch of samples 
that are uncorrelated from the distributions under the evaluation policy.
\citet{fujimoto2019off}
showed that such a discrepancy induces extrapolation error
that severely degrades  performance 
in offline as well as online environments.
The mismatch between 
the data distribution and the target distribution
introduces a selection bias 
since the empirical mean is taken over the batch, 
and 
it also induces a large variance if the importance sampling method is employed in the region where the inverse propensity score is large.
Therefore,  critic learning in offline RL algorithms
has an additional challenge that stems from
limitations around sample collection,
which calls for  methods that improve sample efficiency.

Algorithm \ref{alg:adbcq}
is another variation of critic learning tailored to BCQ, 
which we call action decomposed BCQ (AD-BCQ).
AD-BCQ also introduces small changes to the base algorithm BCQ,
similar in the way AD-DQN modifies DQN.
We augment the steps for learning the dynamics models of the projected MDPs,
and
train the deep Q-networks $\tilde{Q}_{\pi}(\bs, \ba)$ 
that combines sub Q-networks $Q_{\pi_k}(\bs, \ba_k)$.
Since BCQ is an offline RL algorithm, there is no step for collecting online samples.
Therefore, we preprocess the whole collected data $D$
and train all the dynamics models and the reward model before training Q-networks (line 1-4).
For training Q-networks,
we sample a mini-batch $B$ from the dataset $D$,
and modify it as $B_k$ to follow the transitions in the projected MDPs 
for training 
the projected Q-networks $Q_{\pi_k}(\bs, \ba_k)$
and the generative models $G_k(\ba_K|\bs)$ (line 7-9). 
After training projected Q-networks $Q_{\pi_k}(\bs, \ba_k)$,
we train the global deep Q-network including the mixer (line 10).

\section{EXPERIMENTS}
We conduct experiments in 2D-point mass control environment from deep mind control suite \citep{tassa2018deepmind}
and MIMIC-III sepsis treatment environment \citep{komorowski2018artificial,goldberger2000physiobank,johnson2016mimic,killian2020empirical}
for evaluating AD-DQN and AD-BCQ, respectively.
Both RL environments are close to the real-world problems and 
they also have well-defined factored action spaces.
In addition, there are the state-of-the-art algorithms that utilize the linear Q-function decomposition,
which we can compare to see the improvements in sample efficiency.

\begin{figure*}
\centering
\begin{subfigure}[b]{0.31\textwidth}
  \includegraphics[width=1.0\textwidth]{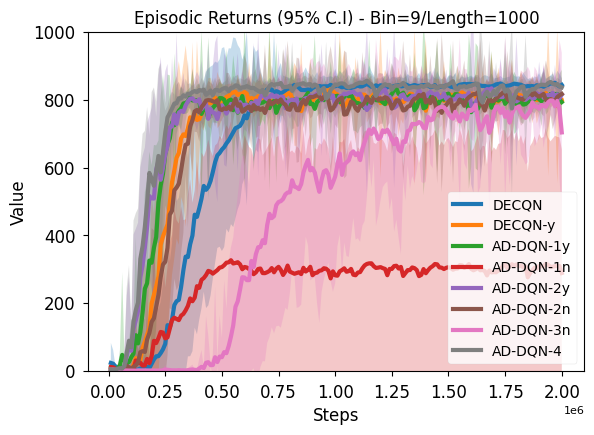}
  \caption{9x9 Action Space}
\end{subfigure}
\begin{subfigure}[b]{0.31\textwidth}
  \includegraphics[width=1.0\textwidth]{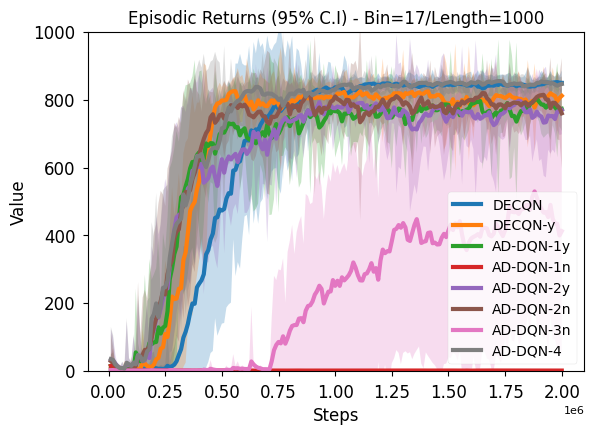}
  \caption{17x17 Action Space}
\end{subfigure}
\begin{subfigure}[b]{0.31\textwidth}
  \includegraphics[width=1.0\textwidth]{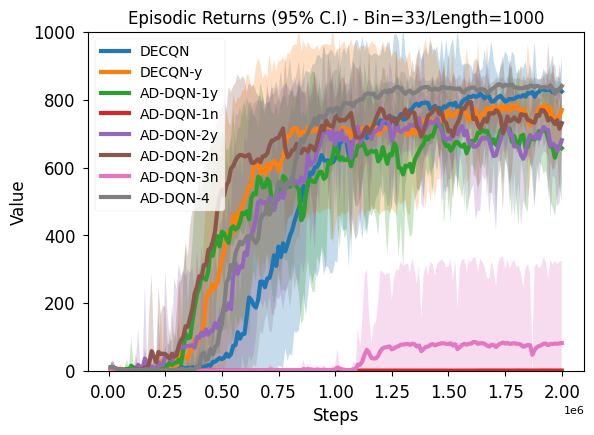}
    \caption{33x33 Action Space}
\end{subfigure}
\caption{Comparing the values in the test environment on 
three action spaces. 
The X-axis shows 2 million steps of training (2000 episodes) and the Y-axis is the value evaluated from a test environment.
The plot aggregates 10 trials with the random seeds from 1 to 10 for training and 1001 to 1010 for testing.
}
\label{fig:2d_point_mass_results}
\end{figure*}

\begin{table}
\footnotesize
\caption{
Algorithm configurations in 2D-Point mass control experiment.
}
\centering
\begin{tabular}{lccc}
\toprule
Algorithms & Weights & Mixer & Data aug\\
\midrule
DECQN & shared & average & no\\
DEQCN-y & shared & average & yes\\
AD-DQN-1y & shared & ReLu & yes\\
AD-DQN-1n & shared & ReLu & no\\
AD-DQN-2y & shared & 2 linear layers & yes\\
AD-DQN-2n & shared & 2 linear layers & no\\
AD-DQN-3n & not shared & 2 linear layers & no\\
AD-DQN-4 & shared & 2 linear layers & yes\\
\bottomrule
 \end{tabular}
 \label{table:2d_point_mass}
\vspace{-1em}
\end{table}

\subsection{2D Point-Mass Control}
We evaluate AD-DQN and 
the decoupled Q-networks (DECQN) \citep{seyde2022solving}
\footnote{We also evaluated variations of DQN and SAC after flattening the factored action space, 
and they all failed to learn Q-functions due to the large number of action labels.
}.
For both algorithms,
we discretized the action space of the continuous actions
to apply DQN-style algorithms.
In 2D point-mass control task,
the state space is 4 dimensional continuous space 
comprised of the position and velocity in the $x$ and $y$ axis on a plane 
,
and 
the action space is 2 dimensional factored action space
that defines $Q_{\pi_x}$ and $Q_{\pi_y}$.
For measuring the performance,
we evaluated average episodic returns
in a separate evaluation environment for 10 trials.

\paragraph{Algorithm Configurations}
In the experiment, we evaluated
various algorithm configurations by varying the hyperparameters. 
Table \ref{table:2d_point_mass} 
summarizes all configurations.
All configurations except for the AD-DQN-3n 
shares the weights of $Q_{\pi_x}$ and $Q_{\pi_y}$.
The mixer combines the two decomposed value functions
by a simple average, passing a 3 layer MLP with ReLu,
or 2 linear layers.
Data augmentation learns the dynamics for each projected MDP
and 
synthesize the samples for training $Q_{\pi_x}$ and $Q_{\pi_y}$.
All algorithm configurations 
select 
the actions from each head of 
$Q_{\pi_x}$ and $Q_{\pi_y}$ and 
if a mixer is used, 
a greedy action is selected by evaluating the outcome of the mixer. 
AD-DQN-4 follows AD-DQN but 
it switches to the pure model-free setting when 
the evaluated value of the current greedy policy is greater than 500.

\paragraph{Results}
Figure \ref{fig:2d_point_mass_results}
shows the episodic returns in 3 action spaces 
by increasing the number of discrete bins from 9 to 33.
First of all, DECQN can be seen as 
one of the algorithm configuration of 
AD-DQN 
with a mixer network 
that aggregates the value with a simple average 
without performing the data augmentation step.
Overall, we see that the weight sharing between
two projected Q-networks performs better,
and the mixer network without the non-linear ReLu activation
performs the best.
The best algorithm configuration is AD-DQN-4,
but other configurations also 
improve the speed of convergence
compared with the baseline DECQN (blue).

\begin{figure*}[t]
\centering
\begin{subfigure}[b]{0.32\textwidth}
  \includegraphics[width=1.0\textwidth]{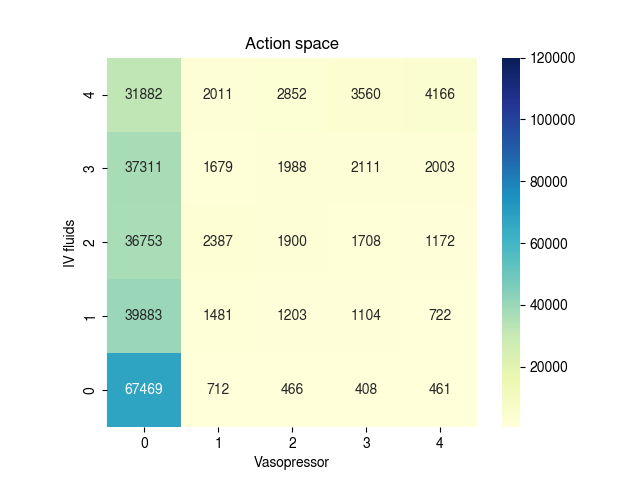}
  \caption{5x5 Action Space}
  \label{f3a}
\end{subfigure}
\begin{subfigure}[b]{0.32\textwidth}
  \includegraphics[width=1.0\textwidth]{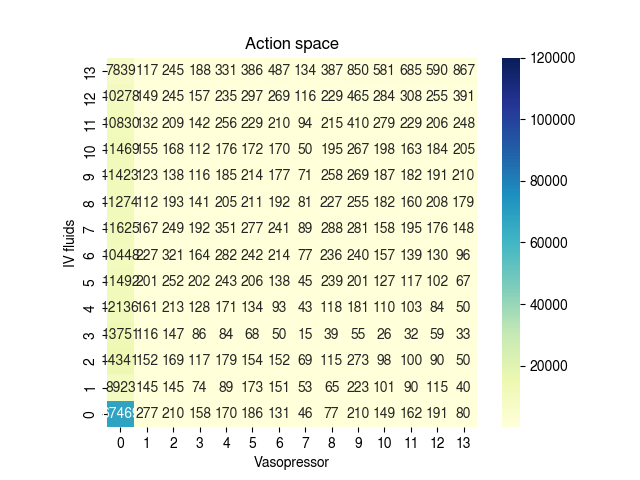}
  \caption{14x14 Action Space}
  \label{f3b}
\end{subfigure}
\begin{subfigure}[b]{0.25\textwidth}
  \includegraphics[width=1.0\textwidth]{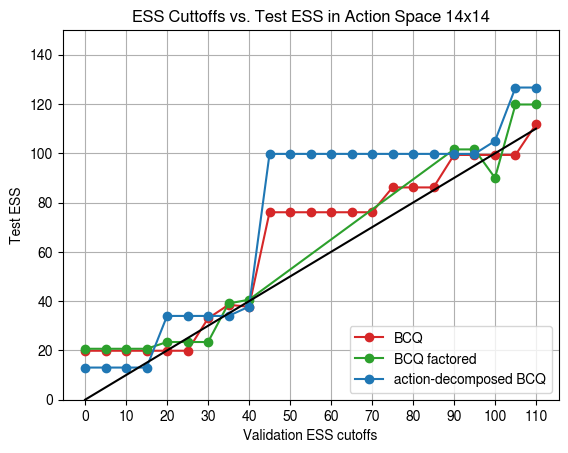}
  \caption{Model selection score 14x14}
  \label{f3c}
\end{subfigure}
\begin{subfigure}[b]{0.29\textwidth}
  \includegraphics[width=1.0\textwidth]{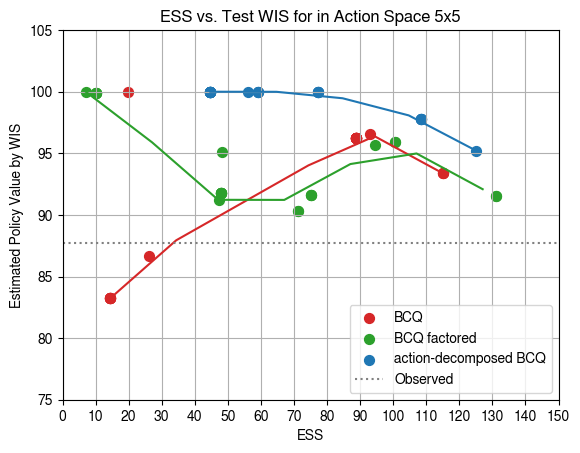}
  \caption{Performance score 5x5}
  \label{f3d}
\end{subfigure}
\begin{subfigure}[b]{0.29\textwidth}
  \includegraphics[width=1.0\textwidth]{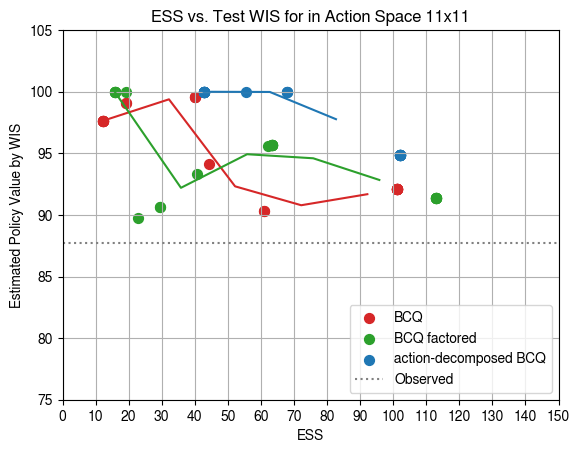}
  \caption{Performance score 11x11}
  \label{f3e}
\end{subfigure}
\begin{subfigure}[b]{0.29\textwidth}
  \includegraphics[width=1.0\textwidth]{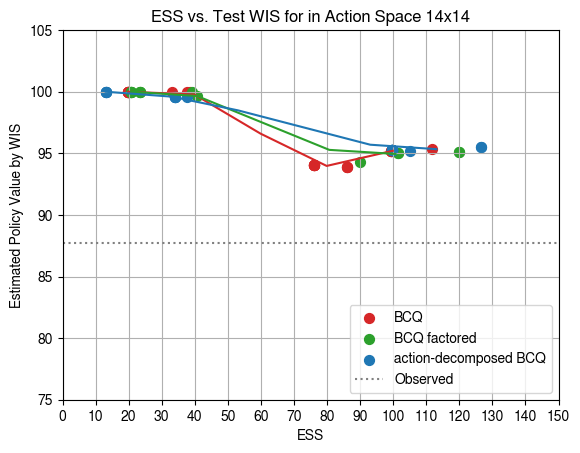}
  \caption{Performance score 14x14}
  \label{f3f}
\end{subfigure}
\caption{
Figure \ref{f3a} and \ref{f3b}
visualize the number of samples for two discrete action spaces
in the training set. 
The offline RL dataset is split into  
70\% training, 15\% for validation, and 15\% test sets.
Figure \ref{f3c} shows the test effective sample size (ESS) score of the policy functions 
selected by the validation ESS score.
Figure \ref{f3d}--\ref{f3f}
show the performance score evaluated from the test set.
The AD-BCQ 
clearly improves upon two baselines, BCQ and factored-BCQ.  
}
\label{fig:mimic_3_results}
\vspace{-1em}
\end{figure*}

\subsection{MIMIC-III Sepsis Treatment}
We evaluate AD-BCQ and factored BCQ \citep{tang2022leveraging}
in the sepsis treatment environment derived 
from MIMIC-III database version 1.4 
\citep{goldberger2000physiobank,johnson2016mimic,komorowski2018artificial}.
Processing the MIMIC-III dataset, 
we obtained a cohort of 177,877 patients and applied 70/15/15 split for training, validation and testing. Each patient has 5 non-time varying demographic features and 33 time varying features collected at 4-hour intervals over 72 hours that make up a discrete time-series. 
The reward is assigned 100 for the discharged state,  -100 for the mortality, and 0 for otherwise.
Among possible choices for the state representation learning 
\citep{killian2020empirical},
we used the approximate information state encoder \citep{subramanian2022approximate}
and obtained 
64 dimensional state space.
The action space is comprised of two continuous variables
for the total volume of intravenous fluid and the amount of vassopressors treated.
We experimented in the action spaces 
by varying the number of bins for the discretization from 5x5 to 14x14.
of those two continuous variables 
by dividing the quantiles  evenly, 
as shown in Figure \ref{f3a} and \ref{f3b}.
When we increase the number of bins, 
we can see that the number of transition samples available for learning decreases, 
which makes offline RL algorithms more challenging.

\paragraph{Performance Measure}
Let $\pi$ be a policy under 
the evaluation and 
$\pi_b$ be a observed behavioral policy in the offline data.
Given an episode $(s_1, a_1, r_1, \ldots, s_L, a_L, r_L)$ of length $L$,
the per-step importance ratio
is 
$\rho_t = {\pi(a_t | s_t)}/{ \pi_b(a_t | s_t})$
and
the cumulative importance ratio 
is 
$\rho_{1:t} = \prod_{t'=1}^{t} \rho_{t'}$.
Given $m$ episodes,
the off-policy evaluation value from the weighted importance sampling (WIS) can be 
computed as follows.
% \begin{equation}
$
    \hat{V}_{\text{WIS}}(\pi)
    =
    1/m
    \sum_{j=1}^m
        \rho^{(j)}_{1:L^{(j)}}
    /
        w_{L^{(j)}}
    \Big(
        \sum_{t=1}^{L^{(j)}} \gamma^{t-1} r_t^{(j)}
    \Big),
$
where
$w_t = {1}/{m}\sum_{j=1}^m \rho_{1:t}^{(j)}$
is the average cumulative importance ratio at time step $t$.
The superscript $(j)$ indicates that 
the length $L^{(j)}$, rewards $r_t^{(j)}$,
and the accumulated importance ratio $\rho_{1:t}^{(j)}$
are computed relative to the $j$-th episode.
The effective sample size (ESS) of the importance sampling can be computed as $
    \text{ESS} = {
    (\sum_{t=1}^N w_t)^2
    }/{
    \sum_{t=1}^N w_t^2
    },
$
where the weights are not normalized \cite{martino2016alternative}.
While computing the weights, we clip the values by $1,000$
and
softened $\pi$
by 
mixing a random policy,
namely
$\hat{\pi}(a|s) = \mathbb{I}_{a=\pi(s)} (1-\epsilon)
+ \mathbb{I}_{a\neq \pi(s)} ({\epsilon}/{|\cA|-1})
$ 
with $\epsilon=0.01$ following \citep{tang2021model}.

\paragraph{Evaluation Results}
We follow the evaluation protocol 
using the open source environment 
offered by \cite{tang2021model,tang2022leveraging}.
For all hyperparameters swept and 20 random trials,
we select the best model with the highest WIS in the validation set
given 
a minimum ESS cutoff value as shown in the black line in Figure \ref{f3c}.
Using the selected model,
we evaluate its WIS and ESS in the test set for the final comparison.
In terms of the model selection score, we see that
AD-BCQ (blue) shows the higher ESS scores compared with the baselines.
Figure \ref{f3d}--\ref{f3f}
show the test WIS scores on the Y-axis, 
and the ESS score on the X-axis. 
The observed value of the clinician's policy employed to 
the patients is 87.75 in the test set, shown in the black dotted lines.
Since this data set has only a single reward
at the end of each episode
and the discounting factor is 1.0 for the evaluation,
this value is the average of the rewards of episodes.

Each algorithm was trained on 20 different random seeds, 
and we see that the contours from AD-BCQ in Figure \ref{f3d}--\ref{f3f}
forms a Pareto frontier of the other two algorithms,
confirming that AD-BCQ learns policy functions that dominates the performance of the baselines
in various action spaces.

\section{Related Works and Discussion}
\citet{tang2022leveraging} presented a method that decomposes Q-function into a linear combination of Q-functions, one per action variable in a factored action space. 
The paper showed a theoretical analysis on the bias and variance of such approximation.
\citet{pitis2022mocoda} considered structural causal model in a factored MDP setting for designing data augmentation algorithms for offline reinforcement learning. 
\citet{luczkow2021structural} studied how to apply structural causal models in reinforcement learning. 
Both approaches failed to capture the causal semantics of the intervention and concluded that the intervention introduces global change.
Based on our study, global influence can be captured by the default dynamics of the system via the \textit{no-op} action in the action model. 

\paragraph{Limitations}
There are several limitations that we want to highlight. 
The proposed approach assumes that we have access to a causal representation or disentangled state variables.
In this work, we do not address how to learn disentangled representation from high-dimensional feature spaces \cite{mitrovic2020representation,brehmer2022weakly,huang2022action}.
In control environments or real-world problems such as medical domains, it is common to see that such representation is already available, and most of the time the action models are also well defined.
Another limitation of our work, especially in the MIMIC-III sepsis treatment experiment, is that the proposed method does not guarantee the improved outcome on the actual situation and the result must be carefully examined by domain experts.

\section{CONCLUSION}
In this paper, we study the decomposable action structure in factored MDPs
with the intervention semantics in causality.
The primary advantage of proposed Q-function decomposition approach 
is shown to improve the sample efficiency in the presence of large discrete action spaces,
one of the major challenges in reinforcement learning.
The theoretical investigation shows that 
we can extend the Q-function decomposition scheme 
from the fully separable MDPs to non-separable MDPs
if the effects of the projected action spaces are non-interacting.
In addition, 
if the underlying Q-function is monotonic, 
a series of local improvements along with the projected action spaces
still returns the optimal policy.

Transferring such findings in ideal settings 
to a more practical DQN-style algorithms, 
we present action decomposed reinforcement learning framework 
that improves the critic learning procedure and demonstrated the improved sample efficiency 
in realistic online and offline environments.

In future work, there are several interesting research directions. 
First, we could jointly learn the causal state representations and 
the causal mechanisms to further improve the sample efficiency.
Second, a more thorough analysis would enhance the theoretical understanding of the Q-function decomposition. 
Last, the extension of problem settings to the `unobserved confounder setting' will present considerable research challenges, 
while also enabling wider adoption to real world applications.

\section*{Acknowledgments}
We thank anonymous reviewers for their feedback to improve the paper. 
We thank the authors of the open source projects that foster rapid prototyping of algorithms and experiments. 
For 2D Point-Mass Control experiment, 
we used the benchmark offered by \citet{tassa2018deepmind}.
For MIMIC-III Sepsis Treatment experiments,  
we follow the earlier work by \citet{komorowski2018artificial}, \citet{killian2020empirical}, \citet{tang2021model}, and \citet{tang2022leveraging} 
for setting up experiment environments and evaluation protocols for 
offline reinforcement learning algorithms.
Especially, we thank to the authors who made baseline algorithms available in open source projects.
We implemented {Action Decomposed DQN} algorithm on top of \texttt{CleanRL} by \citet{huang2022cleanrl}, and 
we extended \texttt{OfflineRL FactoredActions} by \citet{tang2022leveraging}
while implementing Action Decomposed BCQ.
Elliot Nelson and Songtao Lu contributed in this work when they were affiliated with IBM Research.
\bibliography{refs}

%%%%%%%%%%%%%%%%%%%%%%%%%%%%%%%%%%%%%%%%%%%%%%%%%%%%%%%%%%%%
% \newpage
\section*{Checklist}
 \begin{enumerate}
 \item For all models and algorithms presented, check if you include:
 \begin{enumerate}
   \item A clear description of the mathematical setting, assumptions, algorithm, and/or model. [Yes, In Appendix]
   \item An analysis of the properties and complexity (time, space, sample size) of any algorithm. [Yes, In Appendix]
   \item (Optional) Anonymized source code, with specification of all dependencies, including external libraries. [No]
 \end{enumerate}

 \item For any theoretical claim, check if you include:
 \begin{enumerate}
   \item Statements of the full set of assumptions of all theoretical results. [Yes]
   \item Complete proofs of all theoretical results. [Yes]
   \item Clear explanations of any assumptions. [Yes]     
 \end{enumerate}

 \item For all figures and tables that present empirical results, check if you include:
 \begin{enumerate}
   \item The code, data, and instructions needed to reproduce the main experimental results (either in the supplemental material or as a URL). [No]
   \item All the training details (e.g., data splits, hyperparameters, how they were chosen). [Yes, In Appendix]
         \item A clear definition of the specific measure or statistics and error bars (e.g., with respect to the random seed after running experiments multiple times). [Yes, In Appendix]
         \item A description of the computing infrastructure used. (e.g., type of GPUs, internal cluster, or cloud provider). [Yes, In Appendix] 
 \end{enumerate}

 \item If you are using existing assets (e.g., code, data, models) or curating/releasing new assets, check if you include:
 \begin{enumerate}
   \item Citations of the creator If your work uses existing assets. [Yes]
   \item The license information of the assets, if applicable. [No]
   \item New assets either in the supplemental material or as a URL, if applicable. [Not Applicable]
   \item Information about consent from data providers/curators. [Not Applicable]
   \item Discussion of sensible content if applicable, e.g., personally identifiable information or offensive content. [Not Applicable]
 \end{enumerate}

 \item If you used crowdsourcing or conducted research with human subjects, check if you include:
 \begin{enumerate}
   \item The full text of instructions given to participants and screenshots. [Not Applicable]
   \item Descriptions of potential participant risks, with links to Institutional Review Board (IRB) approvals if applicable. [Not Applicable]
   \item The estimated hourly wage paid to participants and the total amount spent on participant compensation. [Not Applicable]
 \end{enumerate}
\end{enumerate}

%%%%%%%%%%%%%%%%%%%%%%%%%%%%%%%%%%%%%%%%%%%%%%%%%%%%%%%%%%%%
\newpage
\onecolumn
\appendix
\section{SUPPLEMENTARY MATERIAL}
\subsection{Missing Proofs}
\propqfunction*
\begin{proof}
The actions in $\cA_k$ only intervene on $\bS'_k$ and 
the rest of the state variables $\bS'\setminus \text{Eff}(\bA_k)$
follow the \textit{no-op} dynamics.
Namely, the state transition follows 
$$\!\!\!\!P\big(\bS_k'|\bS, do(\bA_k)\big)
P\big(\bS_{K+1}'|\bS,\text{Eff}(\bA)\big) 
\!\!\!\!\!\!\!\!
\prod_{i \in [1..K], i \neq k}
\!\!\!\!\!\!\!
P(\bS_i'|\bS)
.
$$
By definition, 
$Q_{\pi_k}(\bs, \ba_k)$
is the value of 
applying action 
$do(\bA^0\!\!=\!\!\ba_k)$ in state $\bS^0\!\!\!=\!\!\bs$ at time step $t\!\!\!=\!\!\!0$,
and applying actions $do\big(\bA^t\!\!\!\!=\!\!\pi_k(\bS^t)\big)$ 
for the remaining time steps $t\!\!=\!\!1..\infty$.
It is easy to rewrite
\begin{align}
&Q_{\pi_k}(\bs, \ba_k)
=
\sum_{\bS^1}P(\bS^1|\bs, do(\ba_k))R(\bs, \ba_k, \bS^1)
+
\sum_{t=1}^{\infty}\sum_{\bS^1..\bS^{t}}
\prod_{j=0}^{t} \gamma^{t} P(\bS^{j+1}|\bS^j, do(\pi_k(\bS^j)))
R(\bS^t, \pi_k(\bS^t),\bS^{t+1})\\
&=
\sum_{\bS^1}P(\bS^1|\bs, do(\ba_k))
\big[
R(\bs, \ba_k, \bS^1)
+
\gamma 
\sum_{t=0}^{\infty}
    \gamma^{t}
    \sum_{\bS^{1}..\bS^{t+1}}
    \prod_{j=0}^{t}
    P(\bS^{j+1}|\bS^{j}, do(\pi_k(\bS^j)))  
    R(\bS^{t}, \pi_k(\bS^{t}), \bS^{t+1})
\big]\\
&=
\sum_{\bS^1}P(\bS^1|\bs, do(\ba_k))
\big[
R(\bs, \ba_k, \bS^1)
+
\sum_{t=1}^{\infty}
\gamma^t
\sum_{\bS^{2}..\bS^{t+1}}
\prod_{j=1}^{t}
P(\bS^{j+1}|\bS^{j}, do(\pi_k(\bS^j)))  
R(\bS^{t}, \pi_k(\bS^{t}), \bS^{t+1})
\big]
\\
&=
\sum_{\bS^1}P(\bS^1|\bs, do(\ba_k))
\big[
R(\bs, \ba_k, \bS^1)
+
\gamma
Q_{\pi_k}(\bS^1, \pi_k(\bS^1))
\big]
\\
&= 
\sum_{\bs'}P(\bs'|\bs, do(\ba_k))
\big[
R(\bs, \ba_k, \bs')
+
\gamma
Q_{\pi_k}(\bs', \pi_k(\bs'))
\big],
\end{align}
as desired.    
\end{proof}
\convergncefpi*
\begin{proof}
The proof follows the convergence of the policy iteration algorithm.

\paragraph{Convergence of policy evaluation}
We show that the policy evaluation steps in line 4-7 
converges to $Q_{\pi}(\bs, \ba_k)$.
Note that $Q_{\pi}(\bs, \ba_k)$ 
is the value of applying action
$[\pi_1{\bs), \ldots, \pi_{k-1}(\bs), \ba_k, \pi_{k+1}(\bs), \ldots, \pi_K(\bs)}]$
in $\bs$ 
followed by the actions following $\pi$ in the subsequent states.
Let $\pi$ be the policy under the evaluation.
We can see that $P(\bS'|\bS, \bA_k)$ in line 5 is $P\big(\bS'|\bS, do(\bA)\big)$
since
\begin{align*}
P(\bS'|\bS, \bA_k)&=
 \mathbb{I}[\bS'_k=\sigma_{\bA_k}(\bS)]
 \prod_{i=1, i\neq k}^{K} \mathbb{I}[\bS'_i=\sigma_{\pi_i(\bS)}(\bS)]
 P\big(\bS'_{K+1}|\bS,\text{Eff}(\bA)\big)\\
 &=
 \mathbb{I}[\text{Eff}(\bA)=\sigma_{\bA}\big(\text{Pre}(\bA)\big)]
 P\big(\bS'_{K+1}|\bS,\text{Eff}(\bA)\big)\\
 &= P\big(\bS'|\bS, do(\bA)\big).
\end{align*}
Then, we can rewrite $\tilde{Q}_{\pi_k}(\bs, \ba_k)$ in Definition 2 as
\begin{align*}
\tilde{Q}_{\pi_k}(\bs, \ba_k)
&=\sum_{\bs'} 
P\big(\bs'|\bs, do(\ba)\big)
[
R(\bs, \ba, \bs') + \gamma \tilde{Q}_{\pi_k}(\bs, \ba)
]
\\
&=\sum_{\bs'} 
P\big(\bs'|\bs, do(\ba)\big)
[
R(\bs, \ba, \bs') + \gamma Q_{\pi}(\bs, \ba)
]\\
&=Q_{\pi}(\bs, \ba)
,
\end{align*}
where 
$\ba_k$ is the projected action in $\cM^k$
and 
$\ba = [\pi_1(\bs), \ldots, \pi_{k-1}(\bs), \ba_k, \pi_{k+1}(\bs), \ldots, \pi_{\bs)}]$
is the action in $\cM$ that follows $\pi$, yet it replaces $\pi_k(\bs)$ with $\ba_k$.
The policy evaluation in the factored policy iteration algorithm is equivalent 
to the policy evaluation in the synchronous policy iteration algorithm.
The only difference is that
we restrict the Q-function to vary only at the projected action variables $\bA_k$ 
instead of all action variables $\bA$.
Therefore, we can conclude that the policy evaluation in the factored policy iteration algorithm converges.

\paragraph{Local convergence of policy improvement}
The convergence of the policy improvement step (line 8-10) follows 
from the fact that the policy improvement step 
changes the factored policy $\pi=[\pi_1,\ldots,\pi_k]$
per block-wise update of the action variables 
in a finite space MDP (line 8).

\paragraph{Global convergence of policy improvement}
If the underlying $Q_{\pi}(\bs, \ba)$ is monotonic \citep{rashid2020monotonic},
then 
the block-wise improvement is equivalent to
the joint improvement,
$Q_{\pi_k}\big(\bs, \pi_k(\ba_k)\big)
> Q_{\pi_k}(\bs, \ba_k)
\iff 
Q_{\pi}\big(\bs, \pi_k(\ba)\big)
> Q_{\pi}(\bs, \ba)
$.
For any iteration of the policy improvement in the factored policy iteration algorithm, 
there exists an equivalent update in the policy iteration algorithm
due to the monotonicity of the Q-function.
Therefore, 
the policy improvement step reaches 
the same fixed point that would have reached by
the non-factored policy iteration.
\end{proof}

\complexitymbfpi*
\begin{proof}
Given $K$ random variables $X^1, \ldots, X^K$ with the domain $X^j \in [0,1]$,
and $N$ independent samples $X^j_1, X^j_2, \ldots, X^j_N$,
let $S^j_N = \sum_{i=1}^N X^j_i$
and $\epsilon > 0$.
Then,
$$
P(|S^j_N - \mathbb{E}[S^j_N]/N|\geq \epsilon) \leq 2 \exp{(-2N \epsilon^2)},
$$
by Hoeffding's inequality,
and 
$$
P( \cup_{j=1}^K |S^j_N/N - \mathbb{E}[S^j_N]/N|\geq \epsilon)
\leq 
\sum_{j=1}^K P(|S^j_N/N - \mathbb{E}[S^j_N]/N|\geq \epsilon)
\leq 2 K \exp{(-2N \epsilon^2)},
$$
by union bound.

We can restate the above inequality as,
with probability $1-\delta$,
$|S^j_N/N - \mathbb{E}[S^j_N]/N| \leq \epsilon$ for all $j=1..K$
if 
$\epsilon = \sqrt{1/(2N) \log{(2K/\delta)}}$ given $N$ and $\delta$
or 
$N \geq 1/(2\epsilon^2) \log{(2K/\delta)}$ given $\epsilon > 0$ and $\delta$.

Given a discrete probability distribution $P(\bX|\bY)$,
we can show the sample complexity of learning
$P$ in the tabular model-based RL setting \citep{agarwal2019reinforcement},
$N \geq {|\bX||\bY|}/{\epsilon^2} \cdot \log ({2 |\bY|}/{\delta}).$
For all vectors $\bY$,
we bound the error of the probability parameters by
\begin{align*}
&
% \max_{\bY} ||\hat{P}(\bX|\bY) - P(\bX|\bY)||_1
% \leq 
% \max_{\bY} \sum_{i=1}^{|\bX|}||\hat{P}(X_i|\bY) - P(X_i|\bY)||_1
% \leq
% \sum_{i=1}^{|\bX|} \max_{\bY}||\hat{P}(X_i|\bY) - P(X_i|\bY)||_1\\
% \leq
P(\cup_{\by} |\hat{P}(X_i|\by) - P(X_i|\by)| \geq \epsilon)
\leq 
|\bY|
P(|\hat{P}(X_i|\by) - P(X_i|\by)| \geq \epsilon)
\leq
2|\bY| \exp(-2N \epsilon^2).
\end{align*}

For each $X_i$,
the number of samples is at least  $1/(2\epsilon^2) \log(2|\bY|/\delta)$,
and we need $\bX N$ samples for estimating all parameters is
$N \geq |\bX| |\bY|/\epsilon^2 \log(2|\bY)/\delta)$.
The desired results can be obtained by 
adjusting the sets $\bX$ and $\bY$ appropriately,
i.e., $\bX = \bS'_{K+1}$ and $\bY = \bS \cup \text{Eff}(\bA)$.
\end{proof}

\newpage
\subsection{Experiment Details}
We used a cluster environment with 2 to 4 CPUs per each run.
\subsubsection{2D Point-Mass Control}
% \subsubsection{Algorithm Hyperparameters}
\paragraph{Hyperparameters}
The following hyperparamters are chosen by grid search.
\begin{itemize}
    \item Adam optimizer learning rate: $1e^{-4}$
    \item discount: $0.99$
    \item target update frequency: 100 steps
    \item target network update rate ($\tau)$: $1.0$
    \item no-op fraction probability: $0.1$
    \item start $\epsilon$: 1.0
    \item end $\epsilon$: 0.1
    \item batch size: 128
    \item learning starts: 20 episodes
    \item model train data size: 200 episodes
\end{itemize}

\paragraph{Q-networks}
\begin{itemize}
\item AD-DQN:
\begin{verbatim}
class ADDQN(nn.Module):
    def __init__(self, state_size, action_size, num_bins):
        super().__init__()
        self.action_size = action_size
        self.num_bins = num_bins
        self.q = nn.Sequential(
            nn.Linear(state_size, 512),
            nn.ReLU(),
            nn.Linear(512, 512),
            nn.ReLU(),
            nn.Linear(512, num_bins + num_bins),
        )
        self.mixer = nn.Sequential(
            nn.Linear(self.action_size, 64),
            nn.ReLU(),
            nn.Linear(64, 1)
        )
        self.apply(self._init_weights)

    def forward(self, x, actions):
        actions = actions.reshape(-1, 2).to(torch.int64)
        x_mask = torch.nn.functional.one_hot(actions[:, 0], num_classes=self.num_bins)
        y_mask = torch.nn.functional.one_hot(actions[:, 1], num_classes=self.num_bins)
        action_mask = torch.cat([x_mask, y_mask], dim=1)
        z = self.q(x) 
        z = z * action_mask
        z = self.mixer(z) 
        return z
\end{verbatim}
\item DECQN:
\begin{verbatim}
class DecQNetwork(nn.Module):
    def __init__(self, state_size, action_size, num_bins):
        super().__init__()
        self.network = nn.Sequential(
            nn.Linear(state_size, 512),        
            nn.ReLU(),
            nn.Linear(512, 512),
            nn.ReLU(),
            nn.Linear(512, action_size),        
            Reshape(-1, action_size//num_bins, num_bins)
        )

    def forward(self, x):
        return self.network(x)
\end{verbatim}
\end{itemize}

\paragraph{Models}
\begin{itemize}
    \item Dynamics:
\begin{verbatim}
class DynamicsModelDelta2l(nn.Module):
    def __init__(self, state_size=2, output_size=1):
        super().__init__()
        self.network = nn.Sequential(
            nn.Linear(state_size, 64),
            nn.Linear(64, 64),
            nn.Linear(64, output_size)
        )
        self.output_size = output_size

    def forward(self, x):
        return self.network(x)

    def evaluate(self, x, prev, noise_variance=0.0001):
        with torch.no_grad():
            output = self(x)
            noise = np.random.normal(0, np.sqrt(noise_variance))
            noise = torch.tensor([noise], dtype=torch.float32).to(output.device)
            return prev + output*(1 + noise)    
\end{verbatim}
    \item Rewards:
\begin{verbatim}
class RewardModel(nn.Module):
    def __init__(self, state_size):
        super().__init__()
        self.network = nn.Sequential(
            nn.Linear(state_size*2 + 2, 64),
            nn.ReLU(),
            nn.Linear(64, 64),
            nn.ReLU(),
            nn.Linear(64, 64),
            nn.ReLU(),
            nn.Linear(64, 1)
        )

    def forward(self, x):
        return self.network(x)    
\end{verbatim}
\end{itemize}

\subsubsection{MIMIC-III Sepsis Treatment}
\paragraph{Formulation of RL Environment}
\begin{itemize}
    \item 
    The 
    sepsis treatment environment 
    \cite{komorowski2018artificial}
    was derived from MIMIC-III database version 1.4
    \cite{goldberger2000physiobank,johnson2016mimic}.
    \item
    The sepsis treatment reinforcement learning environment from MIMIC-III database
    was first formulated by \citet{komorowski2018artificial}.
    \citet{killian2020empirical} enhanced the formulation by 
    improving the representation learning.
    \citet{tang2021model} studied 
    an approach for offline RL evaluation for healthcare settings.
    \citet{tang2022leveraging} showed 
    overall evaluation protocol 
    that we follow
    using the open source environment offered by 
    \citet{tang2022leveraging}
    \footnote{\url{https://github.com/MLD3/OfflineRL_FactoredActions}}.    
\end{itemize}

\paragraph{RL Environment}
\begin{itemize}
    \item We used the code offered by \citet{killian2020empirical}
    with the open sourced projects to process 
    the MIMIC-III database
    \footnote{
        \url{https://github.com/MLforHealth/rl_representations/},
        \url{https://github.com/microsoft/mimic_sepsis}
    }.
    \item The original code was applied to MIMIC-III version 1.3
    and we see that applying the same code to the version 1.4
    results in slightly different statistics.
    \item \citet{tang2022leveraging}
    modified action space of \cite{komorowski2018artificial}
    and also modified the reward 
    that the mortality gets 0 reward. 
    We are reverting this change such that
    the reward is 100 for discharged patients,
    -100 for the mortality and 0 for all other patients.
    \item 
    We obtained a cohort of 177,877 patients and applied 70/15/15 split for training, validation, and testing.
    This split is  exactly the same as \cite{tang2022leveraging}
    since our baseline algorithms are BCQ algorithms in \cite{tang2022leveraging}
    \item Each patient has
    5 non-time varying demographic features and 33 time varying 
    features collected at 4-hour intervals over 72 hours that make up a discrete time-series of the maximum length 20.
    Some episodes are of length much less than 20
    if the episode was terminated either by recovery or mortality.
    The index 0 of the episode is the starting state and 
    the index 19 is the largest step number for the final state.
    \item The state encoding
    is done by the approximate information state encoder  (AIS) \cite{subramanian2022approximate},
    which gives 64 continuous state variables.
    \item The action space is 
    comprised of two continuous variables
    for the total volume of intravenous fluid and the amount of vassopressors treated.
    \item 
    We experimented on action spaces having different discretizations,
    ranging from 5x5, 10x10, 11x11, 13x13, and 14x14.
\end{itemize}

\paragraph{Action Spaces}
\citet{tang2022leveraging} modified the action space
from the original space by \citet{komorowski2018artificial}.
We follow the discretization by \citet{komorowski2018artificial},
by dividing the histogram below
with the equally spaced quantiles.
\begin{figure}[h!]
    \centering
    \includegraphics[width=0.7\textwidth]{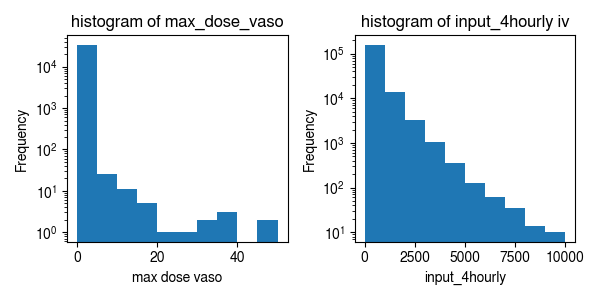}
    \label{fig:adjusted-vaso-iv-hist}
\end{figure}

For example, 
the bins for the action space of size 5 is
[0, 0.07, 0.2, 0.38, 50]
for the vassopressor dose, and
[0.02, 48.0, 150.0, 492.5, 9992.67]
for the IV fluid dose.
The binds for the action space of size 14
is 
[0.00, 0.03, 0.04, 
0.06, 0.09, 0.12, 
0.17, 0.20, 0.23, 0.30, 
0.45, 0.56, 0.90, 50]
for the vassopressor dose and
0.02, 20, 40, 40, 55.92, 80.08, 
120, 186.25, 270, 
380, 512.50, 728.98, 1110, 9992.67
for the IV fluids dose.    

Applying finer grained discretization
leads to larger action spaces as follows.
The $0$-th actions are the no-op action.

\begin{figure}[h!]
    \centering
    \begin{subfigure}[b]{0.33\textwidth}
      \includegraphics[width=1.0\textwidth]{action_space_5.png}
      \caption{Action space of 5x5 discretization}
    \end{subfigure}
    \begin{subfigure}[b]{0.33\textwidth}
      \includegraphics[width=1.0\textwidth]{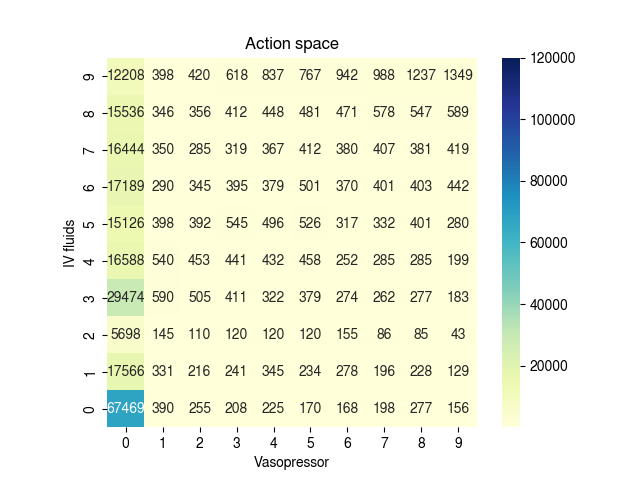}
      \caption{Action space of 10x10 discretization}
    \end{subfigure}
    \begin{subfigure}[b]{0.33\textwidth}
      \includegraphics[width=1.0\textwidth]{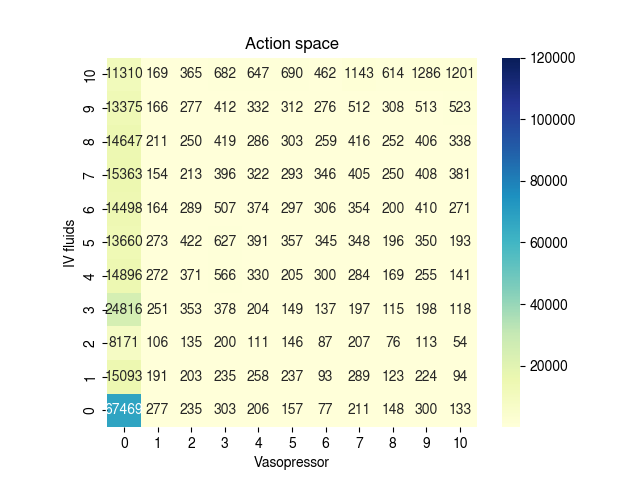}
      \caption{Action space of 11x11 discretization}
    \end{subfigure}
    \begin{subfigure}[b]{0.33\textwidth}
      \includegraphics[width=1.0\textwidth]{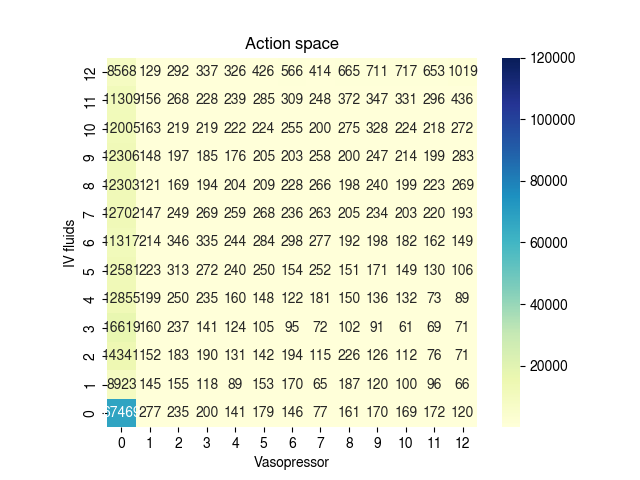}
      \caption{Action space of 13x13 discretization}
    \end{subfigure}
    \begin{subfigure}[b]{0.33\textwidth}
      \includegraphics[width=1.0\textwidth]{action_space_14.png}
      \caption{Action space of 14x14 discretization}
    \end{subfigure} 
    \caption{Discretized action spaces ranging from 5x5 to 14x14}
\end{figure}

\paragraph{AIS State Encoders}
Following the state encoder learning in \cite{tang2022leveraging},
we tried different state dimensions 32, 64, and 128
and different learning rates 0.0005, 0.0001, 0.005, and 0.001.
The best state dimension for all 5 action spaces is 64.
Therefore, the state space of the MDP is 64 dimensional continuous space.
For each configuration, we repeated on 20 different random seeds from 1 to 20,
and selected the best state encoder model showing the minimum validation loss.

\begin{figure}[h!]
    \centering
    \begin{subfigure}[b]{0.28\textwidth}
      \includegraphics[width=1.0\textwidth]{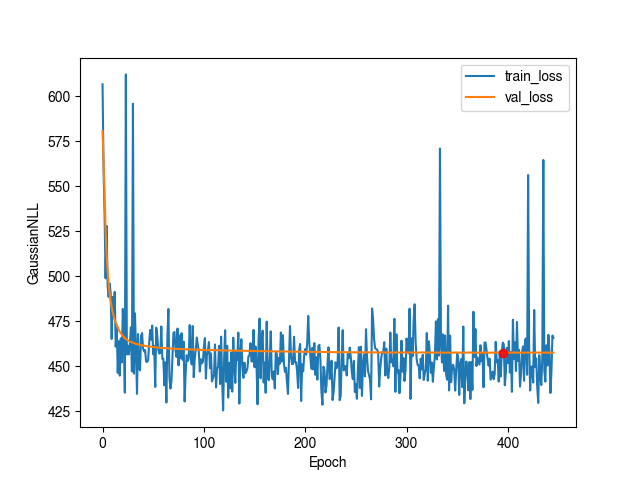}
      \caption{5x5 Loss}
    \end{subfigure}   
    \begin{subfigure}[b]{0.28\textwidth}
      \includegraphics[width=1.0\textwidth]{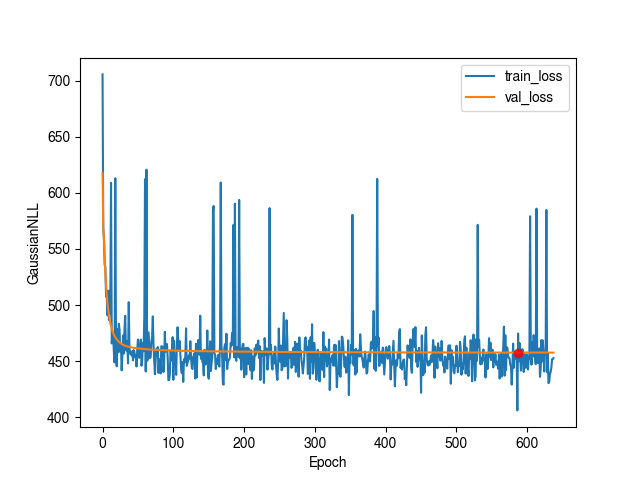}
      \caption{10x10 Loss}
    \end{subfigure}     
    \begin{subfigure}[b]{0.28\textwidth}
      \includegraphics[width=1.0\textwidth]{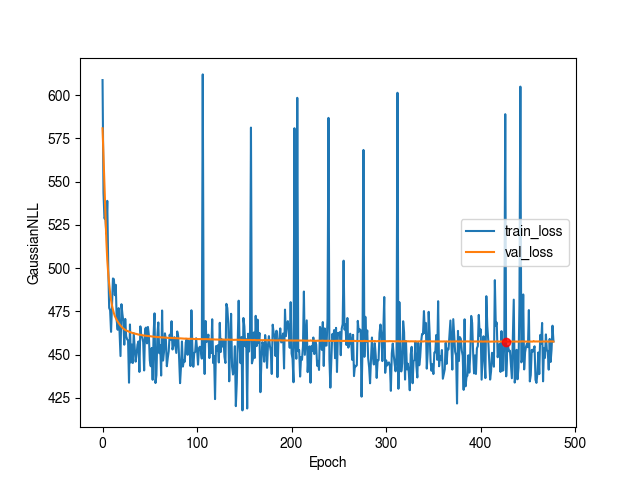}
      \caption{11x11 Loss}
    \end{subfigure}     
    \begin{subfigure}[b]{0.28\textwidth}
      \includegraphics[width=1.0\textwidth]{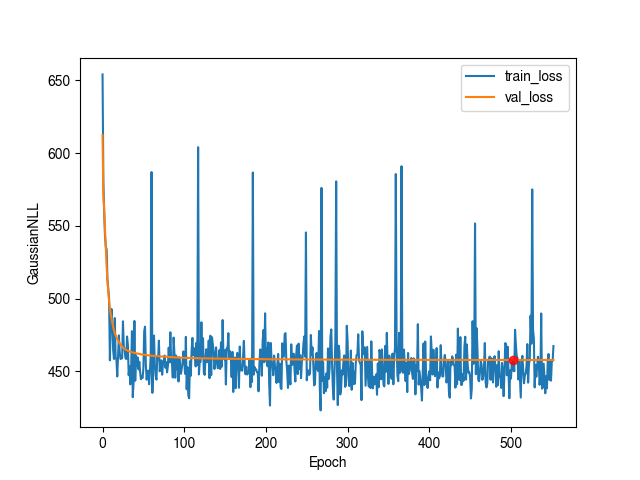}
      \caption{13x13 Loss}
    \end{subfigure}     
    \begin{subfigure}[b]{0.28\textwidth}
      \includegraphics[width=1.0\textwidth]{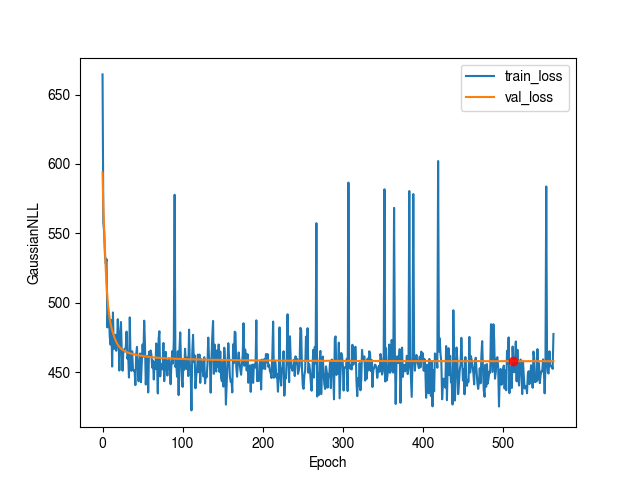}
      \caption{14x14 Loss}
    \end{subfigure}     
    \caption{
    The train and validation loss (Gaussian NLL) 
    for 5 different action spaces ranging from
    5x5 to 14x14.
    For each action space,
    we selected the state encoder model with the minimum
    validation loss
    and encoded the trajectory of all patients.
    }
\end{figure}

\begin{figure}[h!]
    \centering
    \begin{subfigure}[b]{0.28\textwidth}
      \includegraphics[width=1.0\textwidth]{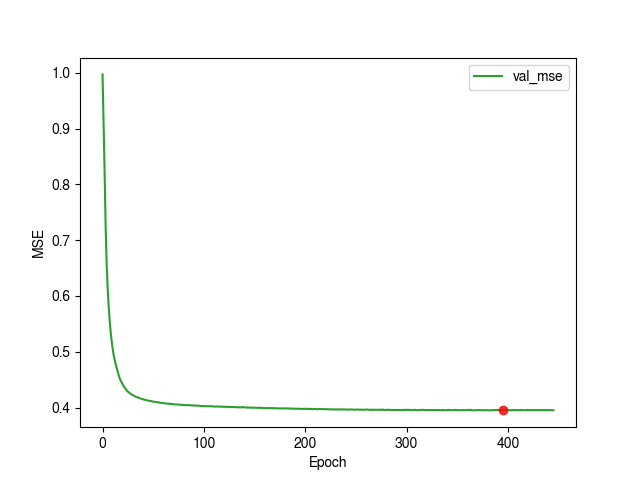}
      \caption{5x5 MSE}
    \end{subfigure}   
    \begin{subfigure}[b]{0.28\textwidth}
      \includegraphics[width=1.0\textwidth]{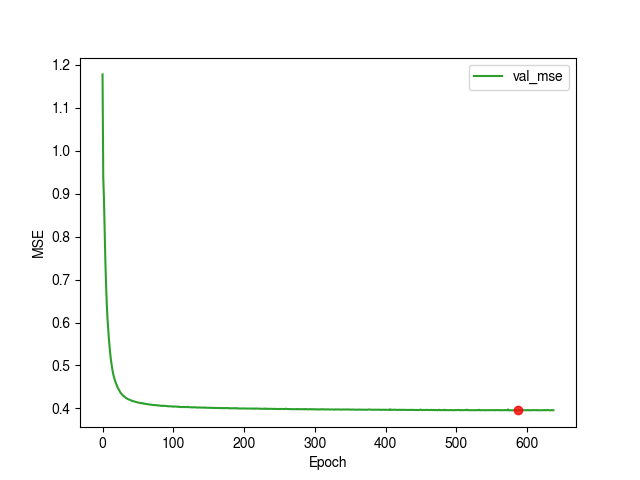}
      \caption{10x10 MSE}
    \end{subfigure}      
    \begin{subfigure}[b]{0.28\textwidth}
      \includegraphics[width=1.0\textwidth]{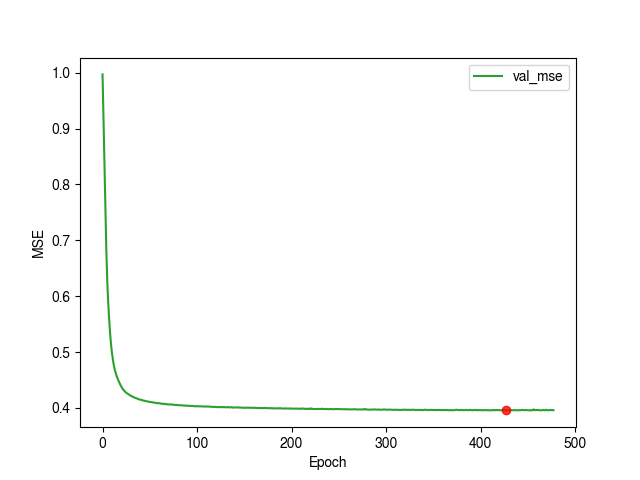}
      \caption{11x11 MSE}
    \end{subfigure}      
    \begin{subfigure}[b]{0.28\textwidth}
      \includegraphics[width=1.0\textwidth]{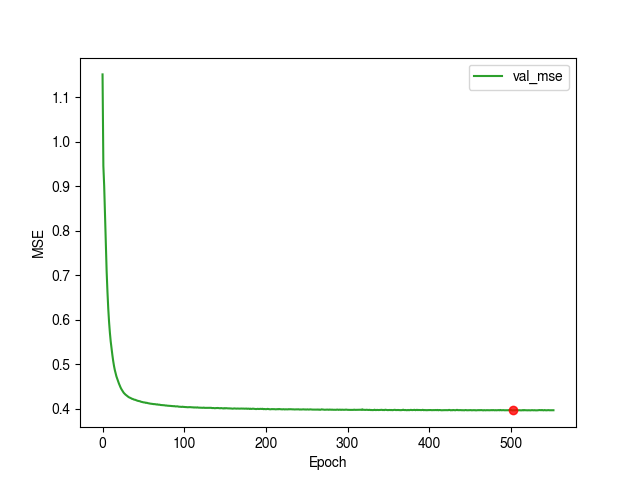}
      \caption{13x13 MSE}
    \end{subfigure}          
    \begin{subfigure}[b]{0.28\textwidth}
      \includegraphics[width=1.0\textwidth]{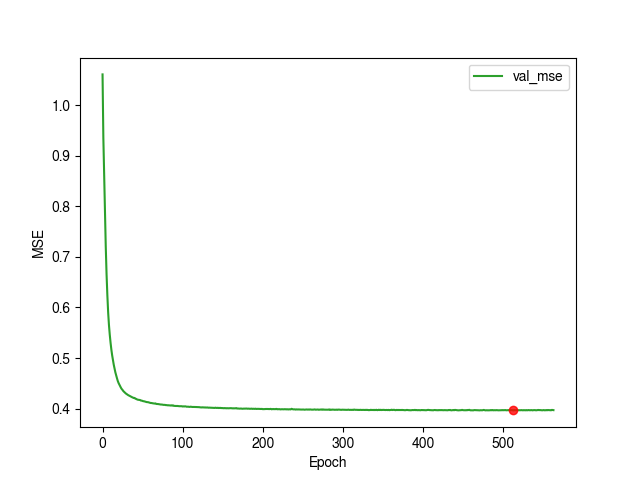}
      \caption{14x14 MSE}
    \end{subfigure}         
    \caption{
    The train and validation mean squared error (MSE)
    for 5 different action spaces ranging from
    5x5 to 14x14.
    }
\end{figure}
\newpage
\paragraph{Evaluation Methods: WIS and ESS}
Let $\pi$ be a policy under evaluation
and 
$\pi_b$ be a observed behavioral policy in the offline data.
Given an episode $(s_1, a_1, r_1, \ldots, s_L, a_L, r_L)$ of length $L$,
the per-step importance ratio
is 
$\rho_t = \frac{\pi(a_t | s_t)}{ \pi_b(a_t | s_t}$
and
the cumulative importance ratio 
is 
$\rho_{1:t} = \prod_{t'=1}^{t} \rho_{t'}$.
Given $m$ episodes,
the off-policy evaluation value from the weighted importance sampling can be 
computed as follows.
\begin{equation}
    \hat{V}_{\text{WIS}}(\pi)
    =
    \frac{1}{m}
    \sum_{j=1}^m
    \frac{
        \rho^{(j)}_{1:L^{(j)}}
    }
    {
        w_{L^{(j)}}
    }
    \Big(
        \sum_{t=1}^{L^{(j)}} \gamma^{t-1} r_t^{(j)}
    \Big),
\end{equation}
where
$w_t = \frac{1}{m}\sum_{j=1}^m \rho_{1:t}^{(j)}$
is the average cumulative importance ratio at time step $t$.
The superscript $(j)$ indicates that 
the length $L^{(j)}$, rewards $r_t^{(j)}$,
and the accumulated importance ratio $\rho_{1:t}^{(j)}$
are computed relative to the $j$-th episode.
The effective sample size of the importance sampling can be computed as follows \cite{martino2016alternative}.
%https://www.nowozin.net/sebastian/blog/effective-sample-size-in-importance-sampling.html
\begin{equation}
    \text{ESS} = \frac{
    (\sum_{t=1}^N w_t)^2
    }{
    \sum_{t=1}^N w_t^2
    },
\end{equation}
where the weights are not normalized.
While computing the weights, we clip the values by $1,000$
and
softened $\pi$
by 
mixing a random policy,
namely
$\hat{\pi}(a|s) = \mathbb{I}_{a=\pi(s)} (1-\epsilon)
+ \mathbb{I}_{a\neq \pi(s)} (\frac{\epsilon}{|\cA|-1})
$ 
with $\epsilon=0.01$
following \cite{tang2021model}.

\begin{itemize}
    \item The offline dataset was split 
    into the train, validation, and test sets
    with 70, 15, and 15 proportions 
    (12989, 2779, and 2791 patients).
    \item Given the dataset, we 
    extract the observed policy (behavior policy)
    using the k nearest neighbor classifier
    implemented in \cite{scikit-learn},
    and we used the default parameter except for 
    $K=100$.
    The stochastic observed policy 
    can be found by 
    predicting the probability of selecting 
    each action in each state encoded by AIS.
    % we stored pi_{b} 
    \item The observed values (the average discounted accumulated rewards) are
    87.74, 87.98, and 87,75 
    for the train, validation, and test set.
    Since this data set has only a single reward
    at the end of each episode
    and the discounting factor is 1.0 for the evaluation,
    the value is the average of the rewards of episodes.
    \item For all hyperparameters swept and random trials,
    we select the best model with the highest validation WIS
    given ESS cutoff value.
    Using the selected model,
    we evaluate its test WIS and test ESS for the final comparison.
\end{itemize}

\paragraph{BCQ hyperparameters}
\begin{itemize}
    \item Adam optimizer learning rate: $3e^{-4}$
    \item weight decay: $1e^{-3}$
    \item discount: $0.99$
    \item target update frequency: 1
    \item Q-learning learning rate ($\tau)$: $0.005$
    \item Q-learning target update: Polyak update
    \item BCQ thresholds: $[0.0, 0.01, 0.05, 0.1, 0.3, 0.5, 0.75, 0.9999]$
\end{itemize}

\paragraph{Q-networks}
\begin{itemize}
\item BCQ: state-dim=64, hidden-dim=128, action-dim=$[5^2, 10^2, 11^2, 13^2, 14^2]$
\begin{verbatim}
class BCQ_Net(nn.Module):
    def __init__(self, state_dim, action_dim, hidden_dim):
        super().__init__()
        self.q = nn.Sequential(
            nn.Linear(state_dim, hidden_dim),
            nn.ReLU(),
            nn.Linear(hidden_dim, hidden_dim),
            nn.ReLU(),
            nn.Linear(hidden_dim, action_dim),
        )
        self.pi_b = nn.Sequential(
            nn.Linear(state_dim, hidden_dim),
            nn.ReLU(),
            nn.Linear(hidden_dim, hidden_dim),
            nn.ReLU(),
            nn.Linear(hidden_dim, action_dim),
        )
    
    def forward(self, x):
        q_values = self.q(x)
        p_logits = self.pi_b(x)
        return q_values, F.log_softmax(p_logits, dim=1), p_logits        
\end{verbatim}
\item BCQ factored: state-dim=64, hidden-dim=128, action-dim=$[5\cdot 2, 10\cdot 2, 11\cdot 2, 13\cdot 2, 14\cdot 2]$
\begin{verbatim}
class BCQf_Net(nn.Module):
    def __init__(self, state_dim, action_dim, hidden_dim):
        super().__init__()
        self.q = nn.Sequential(
            nn.Linear(state_dim, hidden_dim),
            nn.ReLU(),
            nn.Linear(hidden_dim, hidden_dim),
            nn.ReLU(),
            nn.Linear(hidden_dim, action_dim), # vaso + iv
        )
        self.pi_b = nn.Sequential(
            nn.Linear(state_dim, hidden_dim),
            nn.ReLU(),
            nn.Linear(hidden_dim, hidden_dim),
            nn.ReLU(),
            nn.Linear(hidden_dim, action_dim), # vaso + iv
        )
    
    def forward(self, x):
        q_values = self.q(x)
        p_logits = self.pi_b(x)
        return q_values, F.log_softmax(p_logits, dim=-1), p_logits        
\end{verbatim}
\item Action Decomposed BCQ: state-dim=64, hidden-dim=128, action-dim=$[5, 10, 11, 13, 14]$
\begin{verbatim}
class BCQad_Net(nn.Module):
    def __init__(self, state_dim, action_dim, hidden_dim=64):
        super().__init__()

        self.q_embedding = nn.Sequential(
            nn.Linear(state_dim, hidden_dim),
            nn.ReLU(),
            nn.Linear(hidden_dim, hidden_dim),
            # nn.ReLU(),
            # nn.Linear(hidden_dim, hidden_dim),
        )
        self.Q1 = nn.Sequential(
            nn.Linear(hidden_dim, hidden_dim),
            nn.ReLU(),
            nn.Linear(hidden_dim, hidden_dim),
            nn.ReLU(),
            nn.Linear(hidden_dim, action_dim)
        )
        self.Q2 = nn.Sequential(
            nn.Linear(hidden_dim, hidden_dim),
            nn.ReLU(),
            nn.Linear(hidden_dim, hidden_dim),
            nn.ReLU(),
            nn.Linear(hidden_dim, action_dim)
        )
        self.Q = nn.Sequential(
            nn.Linear(action_dim + action_dim, hidden_dim),
            nn.ReLU(),
            nn.Linear(hidden_dim, hidden_dim),
            nn.ReLU(),
            nn.Linear(hidden_dim, action_dim+action_dim)
        )
        self.pi_embedding = nn.Sequential(
            nn.Linear(state_dim, hidden_dim),
            nn.ReLU(),
            nn.Linear(hidden_dim, hidden_dim),
            # nn.ReLU(),
            # nn.Linear(hidden_dim, hidden_dim),
        )
        self.pi1 = nn.Sequential(
            nn.Linear(hidden_dim, hidden_dim),
            nn.ReLU(),
            nn.Linear(hidden_dim, hidden_dim),
            nn.ReLU(),
            nn.Linear(hidden_dim, action_dim)
        )
        self.pi2 = nn.Sequential(
            nn.Linear(hidden_dim, hidden_dim),
            nn.ReLU(),
            nn.Linear(hidden_dim, hidden_dim),
            nn.ReLU(),
            nn.Linear(hidden_dim, action_dim)
        )
        self.pi = nn.Sequential(
            nn.Linear(action_dim + action_dim, hidden_dim),
            nn.ReLU(),
            nn.Linear(hidden_dim, hidden_dim),
            nn.ReLU(),
            nn.Linear(hidden_dim, action_dim+action_dim)
        )

    def forward(self, x, mode):
        if mode == "vaso":
            q = self.q_embedding(x)
            q_values = self.Q1(q)
            p = self.pi_embedding(x)
            p_logits = self.pi1(p)
            return q_values, F.log_softmax(p_logits, dim=-1), p_logits

        if mode == "iv":
            q = self.q_embedding(x)
            q_values = self.Q2(q)
            p = self.pi_embedding(x)
            p_logits = self.pi2(p)
            return q_values, F.log_softmax(p_logits, dim=-1), p_logits

        if mode == "mix":
            with torch.no_grad():
                q = self.q_embedding(x)
                q_values_1 = self.Q1(q)
                q_values_2 = self.Q2(q)
            q_values = self.Q(torch.cat([q_values_1, q_values_2], dim=1))

            with torch.no_grad():
                p = self.pi_embedding(x)
                p_logits1 = self.pi1(p)
                p_logits2 = self.pi2(p)
            p_logits = self.pi(torch.cat([p_logits1, p_logits2], dim=1))

            return q_values, F.log_softmax(p_logits, dim=-1), p_logits

        if mode == "eval":
            with torch.no_grad():
                q = self.q_embedding(x)
                q_values_1 = self.Q1(q)
                q_values_2 = self.Q2(q)
                q_values = self.Q(torch.cat([q_values_1, q_values_2], dim=1))

                p = self.pi_embedding(x)
                p_logits1 = self.pi1(p)
                p_logits2 = self.pi2(p)
                p_logits = self.pi(torch.cat([p_logits1, p_logits2], dim=1))
            return q_values, F.log_softmax(p_logits, dim=-1), p_logits    
\end{verbatim}
\end{itemize}
\newpage
\paragraph{Validation performance - BCQ threshold parameters}
\begin{figure}[b!]
    \centering
    \begin{subfigure}[b]{0.95\textwidth}
      \includegraphics[width=1.0\textwidth]{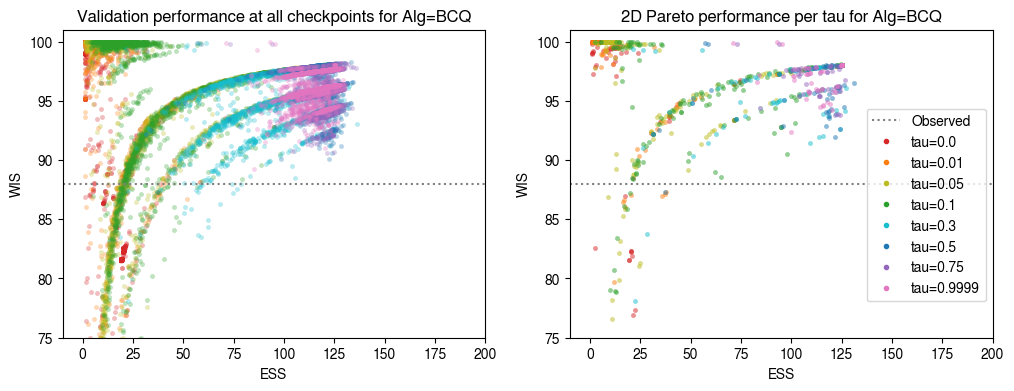}
      \caption{BCQ 5x5}
    \end{subfigure}   
    \begin{subfigure}[b]{0.95\textwidth}
      \includegraphics[width=1.0\textwidth]{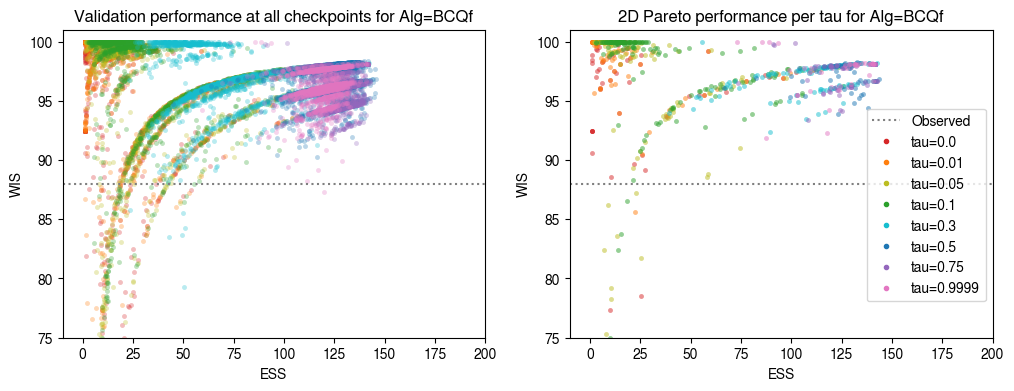}
      \caption{Factored BCQ 5x5}
    \end{subfigure}   
    \begin{subfigure}[b]{0.95\textwidth}
      \includegraphics[width=1.0\textwidth]{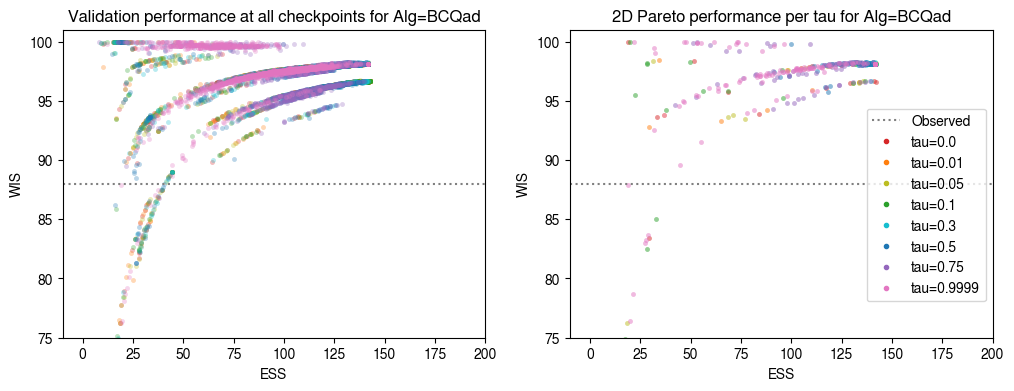}
      \caption{Action Decomposed BCQ 5x5}
    \end{subfigure}       
    \caption{Validation performance of action space 5x5 in terms of WIS and ESS for
    all BCQ threshold parameters and all checkpoints considered during model selection}
    \label{fig:tradeoff-5}
\end{figure}

\newpage
\begin{figure}[h!]
    \centering
    \begin{subfigure}[b]{0.99\textwidth}
      \includegraphics[width=1.0\textwidth]{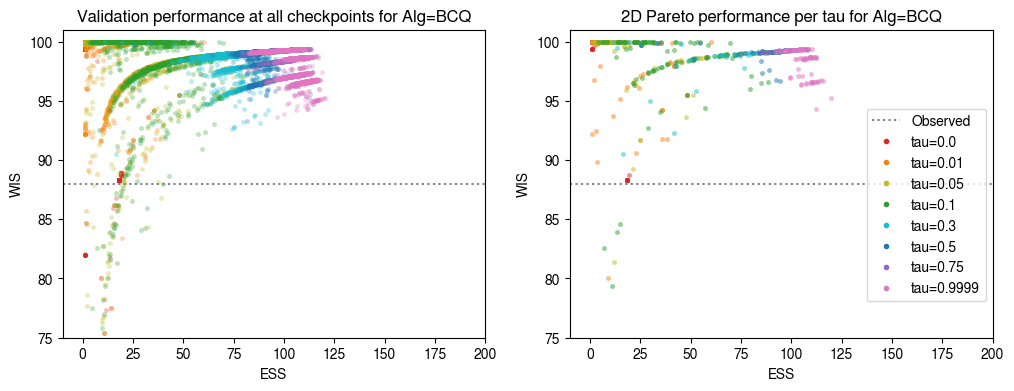}
      \caption{BCQ 10x10}
    \end{subfigure}   
    \begin{subfigure}[b]{0.99\textwidth}
      \includegraphics[width=1.0\textwidth]{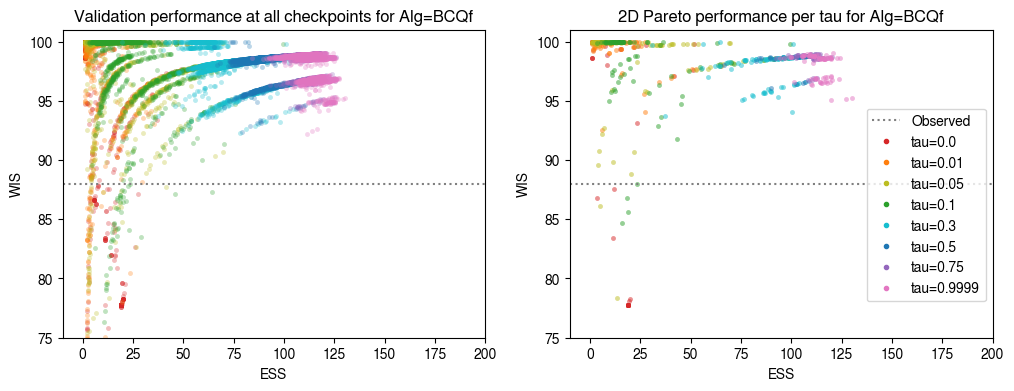}
      \caption{Factored BCQ 10x10}
    \end{subfigure}   
    \begin{subfigure}[b]{0.99\textwidth}
      \includegraphics[width=1.0\textwidth]{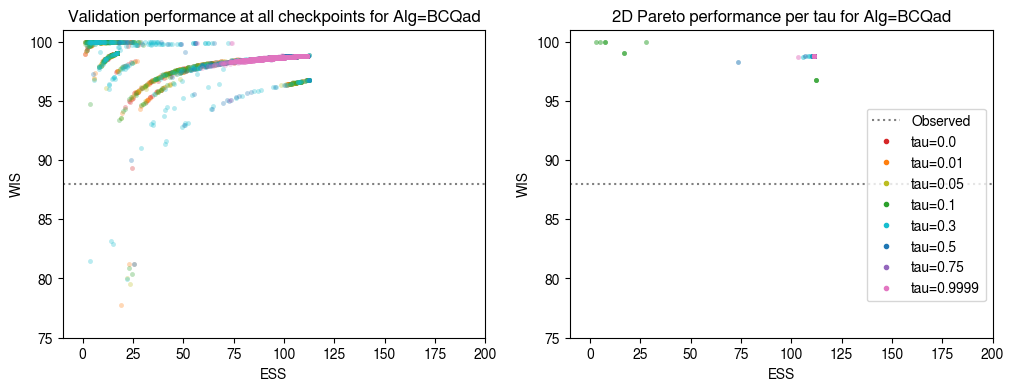}
      \caption{Action Decomposed BCQ 10x10}
    \end{subfigure}       
    \caption{Validation performance of action space 10x10 in terms of WIS and ESS for
    all BCQ threshold parameters and all checkpoints considered during model selection}
    \label{fig:tradeoff-10}
\end{figure}

\newpage
\begin{figure}[h!]
    \centering
    \begin{subfigure}[b]{0.99\textwidth}
      \includegraphics[width=1.0\textwidth]{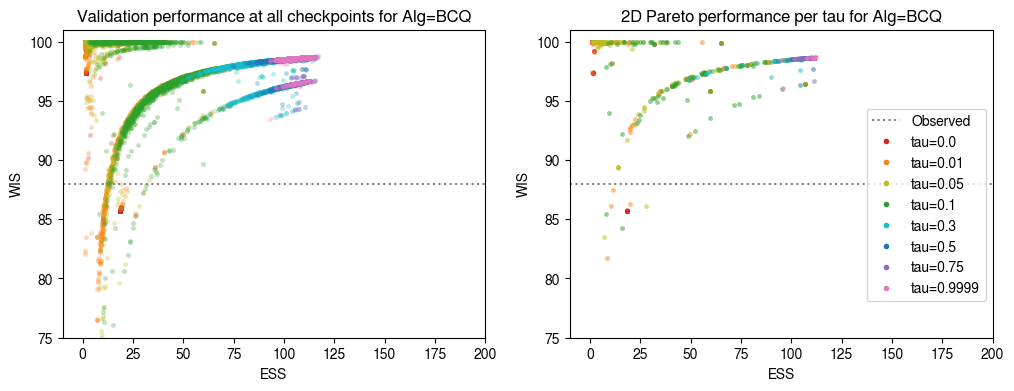}
      \caption{BCQ 11x11}
    \end{subfigure}   
    \begin{subfigure}[b]{0.99\textwidth}
      \includegraphics[width=1.0\textwidth]{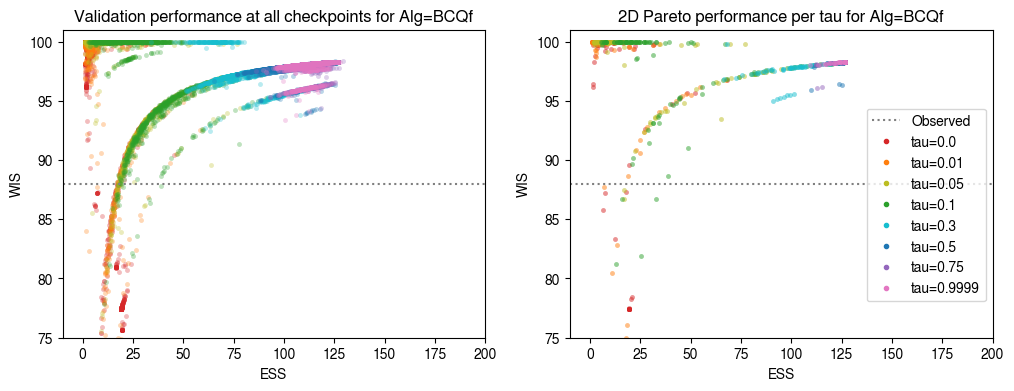}
      \caption{Factored BCQ 11x11}
    \end{subfigure}   
    \begin{subfigure}[b]{0.99\textwidth}
      \includegraphics[width=1.0\textwidth]{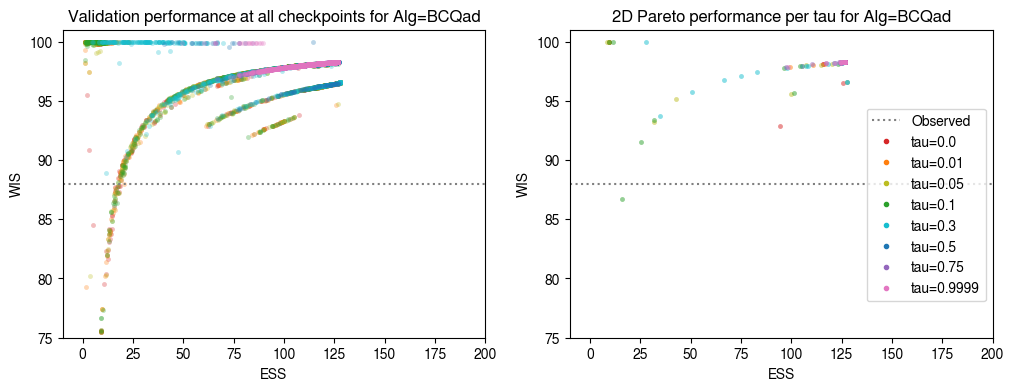}
      \caption{Action Decomposed BCQ 11x11}
    \end{subfigure}       
    \caption{Validation performance of action space 11x11 in terms of WIS and ESS for
    all BCQ threshold parameters and all checkpoints considered during model selection}
    \label{fig:tradeoff-11}
\end{figure}

\newpage
\begin{figure}[h!]
    \centering
    \begin{subfigure}[b]{0.99\textwidth}
      \includegraphics[width=1.0\textwidth]{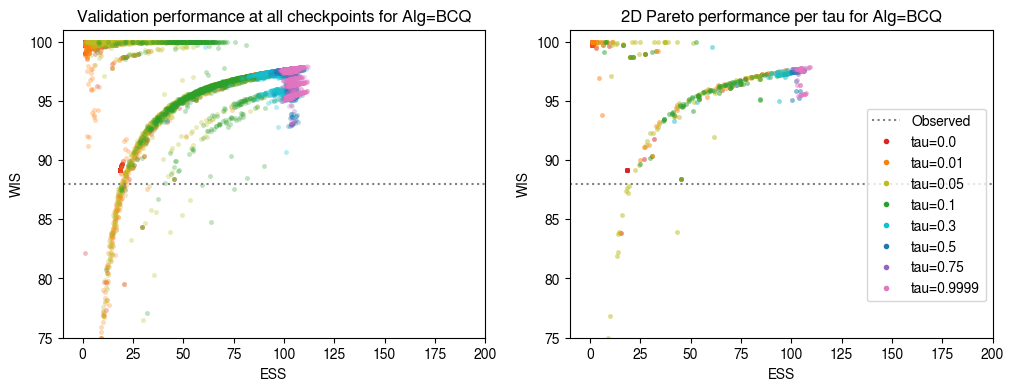}
      \caption{BCQ 13x13}
    \end{subfigure}   
    \begin{subfigure}[b]{0.99\textwidth}
      \includegraphics[width=1.0\textwidth]{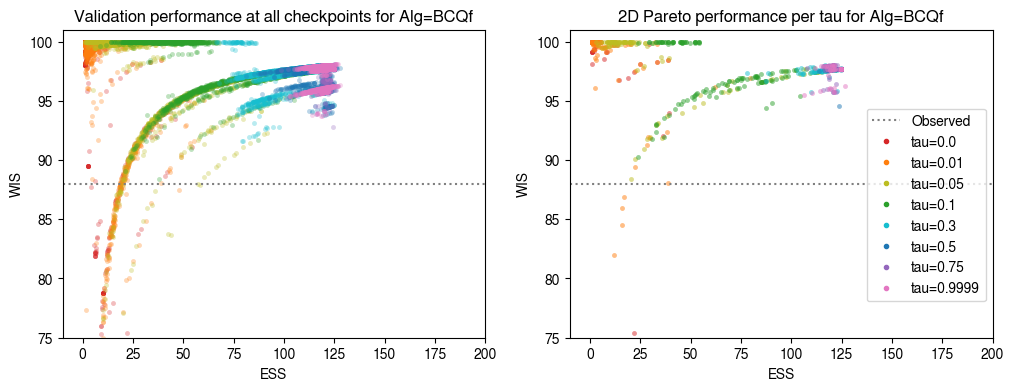}
      \caption{Factored BCQ 13x13}
    \end{subfigure}   
    \begin{subfigure}[b]{0.99\textwidth}
      \includegraphics[width=1.0\textwidth]{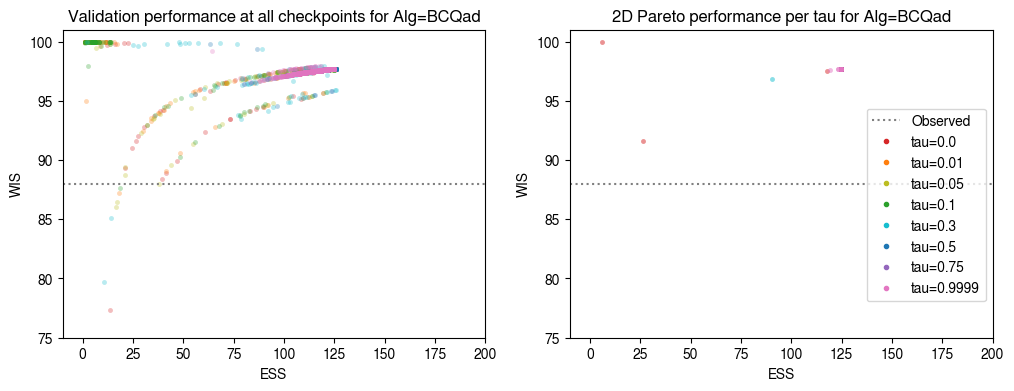}
      \caption{Action Decomposed BCQ 13x13}
    \end{subfigure}       
    \caption{Validation performance of action space 13x13 in terms of WIS and ESS for
    all BCQ threshold parameters and all checkpoints considered during model selection}
    \label{fig:tradeoff-13}
\end{figure}

\newpage
\begin{figure}[h!]
    \centering
    \begin{subfigure}[b]{0.99\textwidth}
      \includegraphics[width=1.0\textwidth]{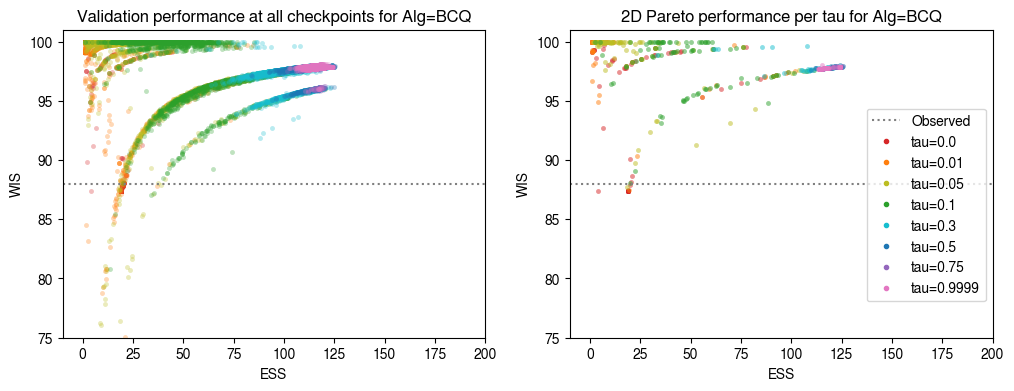}
      \caption{BCQ 14x14}
    \end{subfigure}   
    \begin{subfigure}[b]{0.99\textwidth}
      \includegraphics[width=1.0\textwidth]{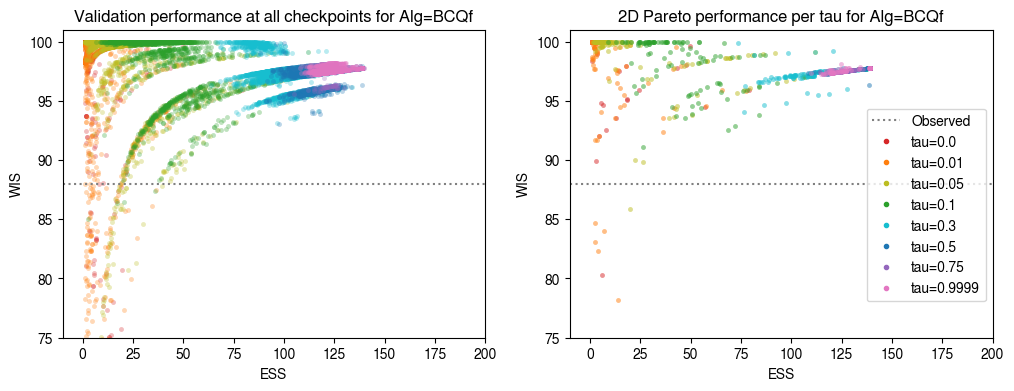}
      \caption{Factored BCQ 14x14}
    \end{subfigure}   
    \begin{subfigure}[b]{0.99\textwidth}
      \includegraphics[width=1.0\textwidth]{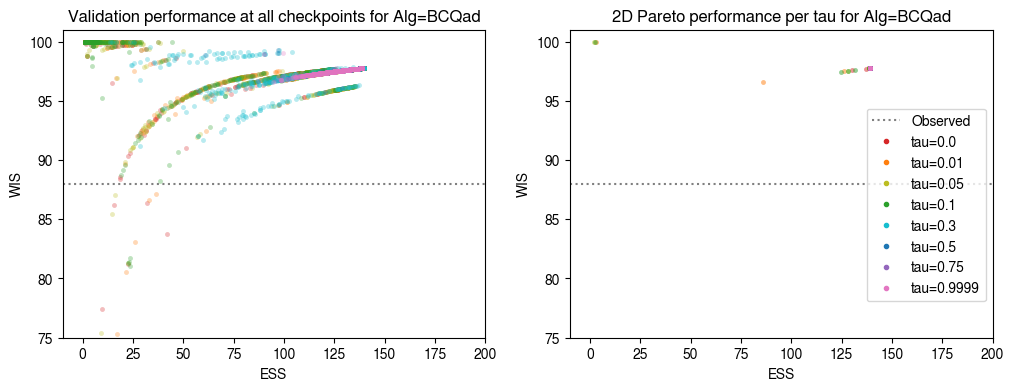}
      \caption{Action Decomposed BCQ 14x14}
    \end{subfigure}       
    \caption{Validation performance of action space 14x14 in terms of WIS and ESS for
    all BCQ threshold parameters and all checkpoints considered during model selection}
    \label{fig:tradeoff-14}
\end{figure}

\newpage
\paragraph{Model Selection with different minimum ESS cutoffs
evaluated on the test set
}
Models are selected such that it must have ESS higher than the cutoff in the validation set.
The following figures show the performance of the selected models
that exceeds the ESS cutoffs in the test set.

\begin{figure}[h!]
    \centering
    \begin{subfigure}[b]{0.43\textwidth}
      \includegraphics[width=1.0\textwidth]{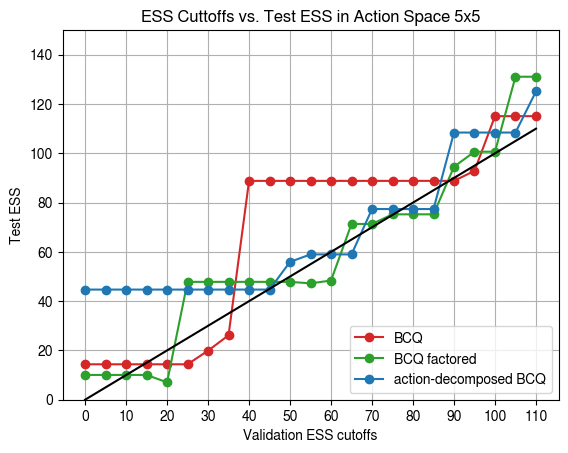}
      \caption{Action space of 5x5}
    \end{subfigure}
    \begin{subfigure}[b]{0.43\textwidth}
      \includegraphics[width=1.0\textwidth]{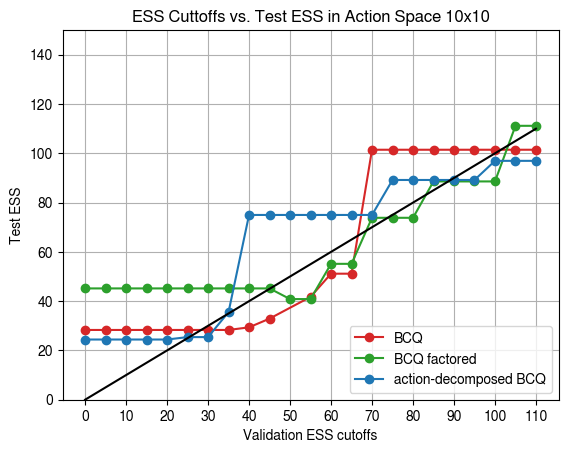}
      \caption{Action space of 10x10}
    \end{subfigure}
    \begin{subfigure}[b]{0.43\textwidth}
      \includegraphics[width=1.0\textwidth]{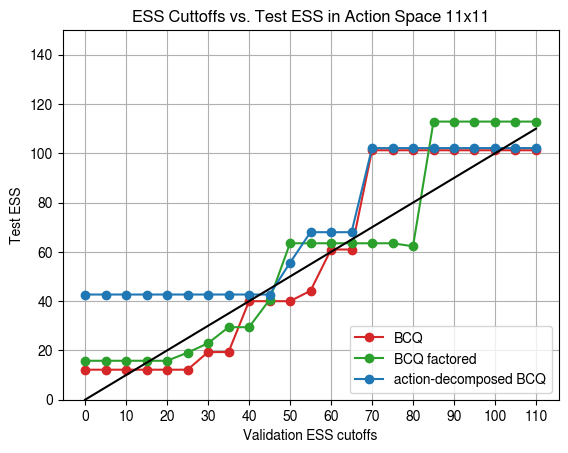}
      \caption{Action space of 11x11}
    \end{subfigure}
    \begin{subfigure}[b]{0.43\textwidth}
      \includegraphics[width=1.0\textwidth]{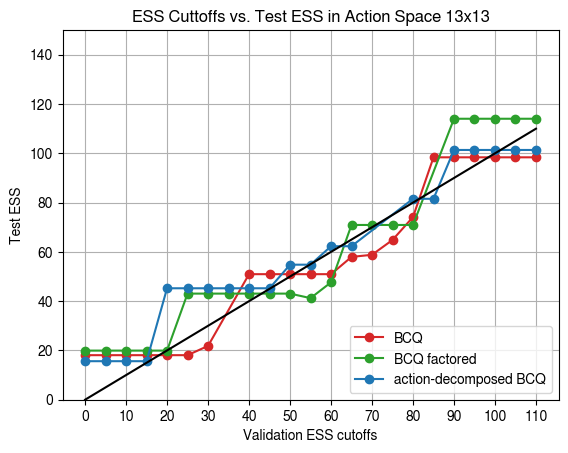}
      \caption{Action space of 13x13}
    \end{subfigure}
    \begin{subfigure}[b]{0.43\textwidth}
      \includegraphics[width=1.0\textwidth]{test-cutoff-ess-14.png}
      \caption{Action space of 14x14}
    \end{subfigure} 
    \caption{Model Selection Scores:
    the X-axis is minimum ESS cutoff values and
    the Y-axis is ESS computed in the test set
    using the selected model based on the highest validation WIS scores
    in action spaces ranging from 5x5 to 14x14}
\end{figure}

\newpage
\textbf{
Model Selection with different minimum ESS cutoffs
evaluated on the validation set}
Models are selected such that it must have ESS higher than the cutoff in the validation set.
The following figures show the performance of the selected models
that exceeds the ESS cutoffs in the validation set.
\begin{figure}[h!]
    \centering
    \begin{subfigure}[b]{0.43\textwidth}
      \includegraphics[width=1.0\textwidth]{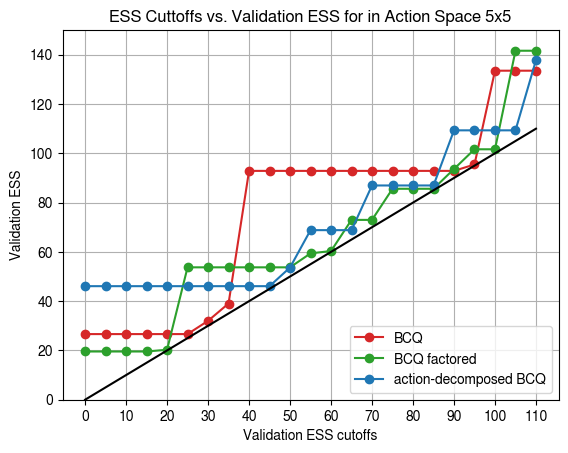}
      \caption{Action space of 5x5}
    \end{subfigure}
    \begin{subfigure}[b]{0.43\textwidth}
      \includegraphics[width=1.0\textwidth]{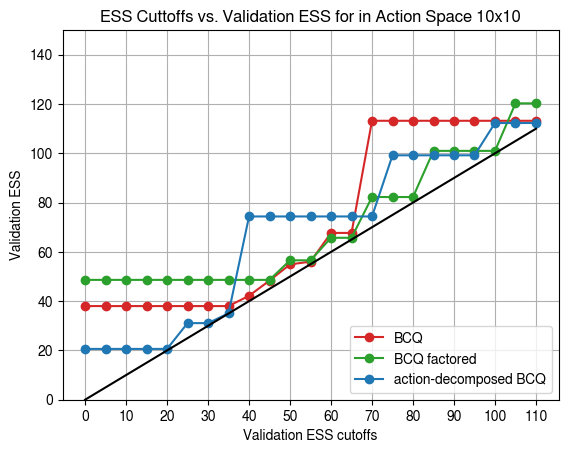}
      \caption{Action space of 10x10}
    \end{subfigure}
    \begin{subfigure}[b]{0.43\textwidth}
      \includegraphics[width=1.0\textwidth]{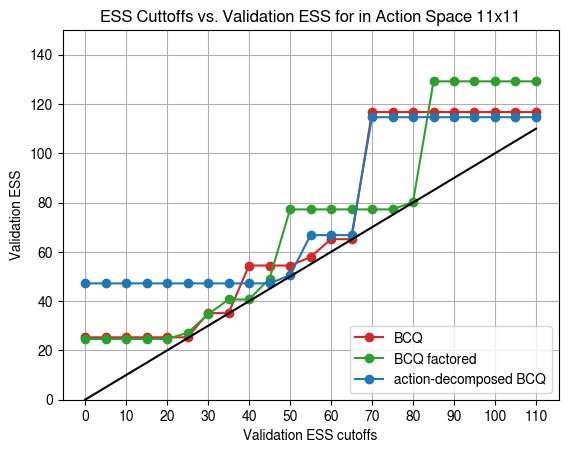}
      \caption{Action space of 11x11}
    \end{subfigure}
    \begin{subfigure}[b]{0.43\textwidth}
      \includegraphics[width=1.0\textwidth]{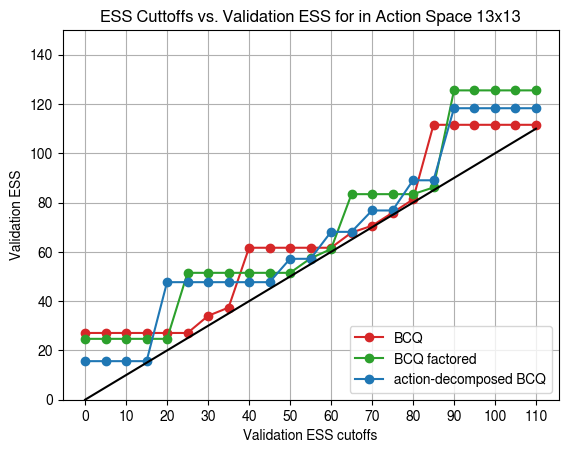}
      \caption{Action space of 13x13}
    \end{subfigure}
    \begin{subfigure}[b]{0.43\textwidth}
      \includegraphics[width=1.0\textwidth]{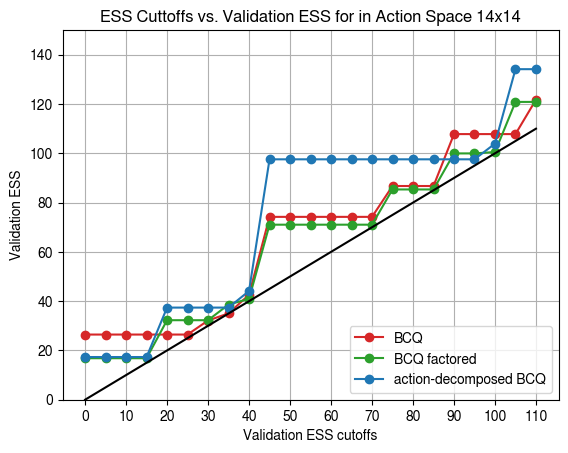}
      \caption{Action space of 14x14}
    \end{subfigure} 
    \caption{Model Selection Scores:
    the X-axis is minimum ESS cutoff values and
    the Y-axis is ESS computed in the validation set
    using the selected model based on the highest validation WIS scores
    in action spaces ranging from 5x5 to 14x14}
\end{figure}

\newpage
\paragraph{Performance Scores Evaluated on Test Set}
The blue curves from \texttt{AD-BCQ} dominates other two baselines.
% ESS vs. WIS for the selected models   -- csv file stores 
% fig 17 R
% done
\begin{figure}[h!]
    \centering
    \begin{subfigure}[b]{0.43\textwidth}
      \includegraphics[width=1.0\textwidth]{test-performance-5.png}
      \caption{Action space of 5x5}
    \end{subfigure}
    \begin{subfigure}[b]{0.43\textwidth}
      \includegraphics[width=1.0\textwidth]{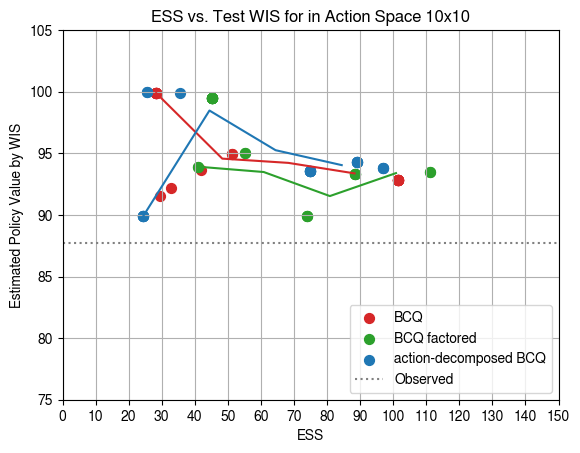}
      \caption{Action space of 10x10}
    \end{subfigure}
    \begin{subfigure}[b]{0.43\textwidth}
      \includegraphics[width=1.0\textwidth]{test-performance-11.png}
      \caption{Action space of 11x11}
    \end{subfigure}
    \begin{subfigure}[b]{0.43\textwidth}
      \includegraphics[width=1.0\textwidth]{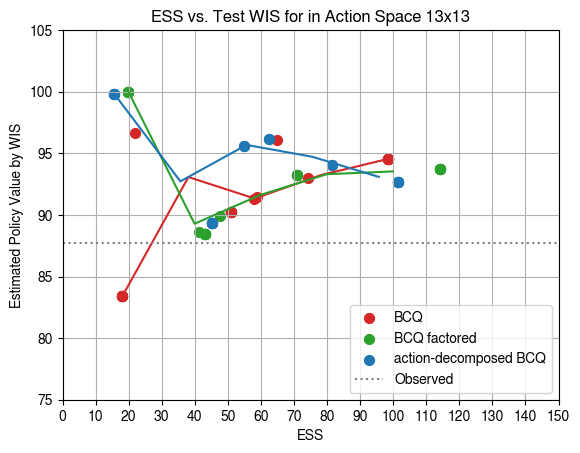}
      \caption{Action space of 13x13}
    \end{subfigure}
    \begin{subfigure}[b]{0.43\textwidth}
      \includegraphics[width=1.0\textwidth]{test-performance-14.png}
      \caption{Action space of 14x14}
    \end{subfigure} 
    \caption{Performance Scores of Selected Models:
    Each point represents the performance score of the selected model.
    the X-axis is ESS and the Y-axis is WIS.
    The figure shows the Pareto frontiers of the test performance
    in the action spaces ranging from 5x5 to 14x14}
\end{figure}
\end{document}